\documentclass{article}
\usepackage{amsthm,mathtools}
\DeclareMathOperator*{\argmin}{arg\,min}
\DeclareMathOperator*{\argmax}{arg\,max}
\newtheorem{lemma}{Lemma}

\usepackage[toc,page]{appendix}
\DeclarePairedDelimiter\abs{\lvert}{\rvert}%
\usepackage{algorithm,algorithmic}
\usepackage{xcolor}
\PassOptionsToPackage{numbers, compress}{natbib}

\usepackage[preprint]{neurips_2020}

\usepackage[preprint]{neurips_2020}
\usepackage{amsmath}
\usepackage{adjustbox}
\usepackage{array}
\usepackage{booktabs}
\usepackage{dsfont}
\usepackage{multirow}
\usepackage{hhline}
\usepackage[utf8]{inputenc} 
\usepackage[T1]{fontenc}    
\usepackage{hyperref}
\usepackage{url}            
\usepackage{booktabs}       
\usepackage{amsfonts}       
\usepackage{nicefrac}       
\usepackage{microtype}      
\usepackage{amsthm}
\theoremstyle{definition}
\newtheorem{definition}{Definition}[section]
\theoremstyle{theorem}
\newtheorem{theorem}{Theorem}
\newtheorem{proposition}{Proposition}
\usepackage{float}
\usepackage[caption = false]{subfig}

\makeatletter
\newcommand*{\centerfloat}{%
  \parindent \z@
  \leftskip \z@ \@plus 1fil \@minus \textwidth
  \rightskip\leftskip
  \parfillskip \z@skip}
\makeatother

\title{Learning Unbiased Representations via Rényi Minimization}

\author{%
 Vincent Grari \\
 Sorbonne Université LIP6/CNRS\\
 Paris, France\\
 \texttt{vincent.grari@lip6.fr} \\
  \And
   Oualid El Hajouji \\
   Ecole polytechnique \\
   Palaiseau, France \\
   \texttt{oualid.el-hajouji@polytechnique.edu} \\
   \AND
   Sylvain Lamprier \\
   Sorbonne Université LIP6/CNRS \\
   Paris, France \\
   \texttt{sylvain.lamprier@lip6.fr} \\
   \And
   Marcin Detyniecki \\
   AXA REV Research \\
   Paris, France \\
   \texttt{marcin.detyniecki@axa.com} 
}

\begin{document}

\maketitle

\begin{abstract}
In recent years, significant work has been done to include fairness constraints in the training objective of machine learning algorithms.
Many state-of the-art algorithms tackle this challenge by learning a fair representation which captures all the relevant information to predict the output $Y$ while not containing any information about a sensitive attribute $S$. In this paper, we propose an adversarial algorithm to learn unbiased representations via the Hirschfeld-Gebelein-Renyi (HGR) maximal correlation coefficient. We leverage recent work which has been done to estimate this coefficient by learning deep neural network transformations and use it as a min-max game to penalize the intrinsic bias in a multi dimensional latent representation. Compared to other dependence measures, the HGR coefficient captures more information about the non-linear dependencies with the sensitive variable, making the algorithm more efficient in mitigating bias in the representation. We empirically evaluate and compare our approach and demonstrate significant improvements over existing works in the field.


\end{abstract}

\section{Introduction}

\footnotetext[1]{A link to our python code is available here: \url{https://github.com/fairness-adversarial/unbiased_representations_renyi}.} 
This recent decade, deep learning models have shown very competitive results 
by learning 
representations that 
capture relevant information for the learning task. 
However, the representation learnt by the deep model may contain some bias from the training data. This bias can be intrinsic to the training data, and may therefore induce a generalisation problem due to a distribution shift between training and testing data. For instance, the color bias in the colored MNIST data set \cite{kim2019learning} can make models focus on the color of a digit rather than its shape for the classification task. The bias can also go beyond training data, so that inadequate representations can perpetuate or even reinforce some society biases~\cite{Bolukbasi2016} (e.g. gender or age). Since the machine learning models have far-reaching consequences in our daily lives (credit rating, insurance pricing, recidivism score, etc.), we need to make sure that the representation data contains as little bias as possible.
A naive method to mitigate bias could be to simply remove sensitive attributes from the training data set~\cite{Pedreshi2008}. However, this concept, known as "fairness through unawareness", is highly insufficient because any other non-sensitive attribute might indirectly contain significant sensitive information reflected in the deep learning representation. For example, the height of an adult could provide a strong indication about the gender.
A new research field has emerged to find solutions to this problem: fair machine learning. Its overall objective is to ensure that the prediction model is not dependent on a sensitive attribute~\cite{Zafar2017mechanisms}. Many recent papers tackle this challenge using an adversarial neural architecture, which can successfully mitigate the bias.
We distinguish two adversarial mitigation families with, first, prediction retreatment methods where an adversarial neural network
encourages bias mitigation  
on the output prediction; 
and, second, fair representation methods where an adversarial 
mitigates the bias on an intermediary latent representation. Recent papers have shown that the fair adversarial representation can tend to give better results in terms of prediction accuracy while remaining fair in complex real-world scenarios \cite{adel2019one}. In this paper, we 
propose a new fair representation architecture by leveraging the recent 
Renyi neural estimator, previously used in a prediction retreatment algorithm \cite{grari2019fairness} which 
obtained very competitive results on various real-world data sets.

The contributions of this paper are:
\begin{itemize}
    \item We provide a theoretical analysis of the HGR estimation by neural network which was not provided in \cite{grari2019fairness};
    \item We propose a neural network architecture which creates a fair representation by minimizing the HGR coefficient. 
    The HGR network is trained  
    to discover non-linear transformations 
    between the multidimensional latent representation and the sensitive feature;
    \item We demonstrate empirically that our neural HGR-based approach is able to identify the optimal transformations with multidimensional features and present very competitive results for fairness learning with continuous sensitive features.  
\end{itemize}



\section{Related Work}
\label{sec:related_work}
Significant work has been done in the field of fair machine learning recently, in particular when it comes to quantifying and mitigating undesired bias. For the mitigation approaches, three distinct strategy groups exist.
While   
pre-processing \cite{kamiran2012data,bellamy2018ai,calmon2017optimized} and post-processing  \cite{hardt2016equality,chen2019fairness} approaches respectively  act on the input or the output of a classically trained predictor,  
in-processing approaches mitigate the undesired bias directly during the training phase \cite{Zafar2017mechanisms,celis2019classification, zhang2018mitigating,wadsworth2018achieving,louppe2017learning}. In this paper we focus on in-processing fairness, which proves to be the most powerful framework for settings where acting on the training process is an option.

Among the in-processing approaches, some of them, referred to as prediction retreatment, aim at directly modifying the prediction output by adversarial training. To ensure independence between the output and the sensitive attribute, Zhang et al. \cite{zhang2018mitigating} feed the prediction output as input to an adversary network (upper right in Figure 1 
in appendix), whose goal is to predict the sensitive attribute, and update the predictor weights to fool the adversary. Grari et al. \cite{grari2019fairness} minimize the HGR correlation between the prediction output and the sensitive attribute in an adversarial learning setting (middle right in 
Figure 1
in appendix).

On the other hand, several research sub-fields in the in-processing family tackle the problem of learning unbiased representations. Domain adaptation \cite{daume2006domain, blitzer2006domain} and domain generalization \cite{muandet2013domain, li2017deeper} consist in learning representations that are unbiased with respect to a source distribution, and can therefore generalize to other domains. Some of the works in these fields involve the use of adversarial methods \cite{ganin2014unsupervised, ganin2016domain, tzeng2017adversarial}, close to our work. 
Several strategies mitigate bias towards a sensitive attribute through representation. One approach \cite{zemel2013learning} relies on a discriminative clustering model to learn a multinomial representation that removes information regarding a binary sensitive attribute. A different approach \cite{alvi2018turning} consists in learning an unbiased representation by minimizing a confusion loss. Invariant representations can also be learnt using Variational Auto-Encoders \cite{kingma2013auto}, by adding a mutual information penalty term \cite{moyer2018invariant}. Adel et al. \cite{adel2019one} learn a fair representation by inputting it to an adversary network, which is prevented from predicting the sensitive attribute (upper left in 
Figure 1
in appendix). Other papers minimize the mutual information between the representation and the sensitive attribute: Kim et al. \cite{kim2019learning} rely on adversarial training with a discriminator detecting the bias, while Ragonesi et al. \cite{ragonesi2020learning} rely on an estimation by neural network of mutual information \cite{belghazi2018mutual} (lower left in Figure 1 
in appendix). 

\section{Problem Statement}
\label{sec:Problem_statement}
Throughout this document, we consider a supervised machine learning algorithm for regression or classification problems. The training data consists of $n$ examples ${(x_{i},s_{i},y_{i})}_{i=1}^{n}$, where $x_{i} \in \mathbb{R}^{p}$ is the feature vector with $p$ predictors of the $i$-th example, $s_i$ is its continuous sensitive attribute and $y_{i}$ its continuous or discrete outcome. We address a common objective in fair machine learning, \emph{Demographic Parity}, which ensures that the sensitive attribute $S$ is independent of the prediction $\widehat{Y}$.

\subsection{Metrics for Continuous Statistical Dependence}

In order to assess this fairness definition in the continuous case, it is essential to look at the concepts and measures of statistical dependence. 
Simple ways of measuring dependence are Pearson's rho, Kendall's tau or Spearman's rank. Those types of measure have already been used in fairness, with the example of mitigating the conditional covariance for categorical variables~\cite{Zafar2017mechanisms}. However, the major problem with these measures is that they only capture a limited class of association patterns, like linear or monotonically increasing functions. For example, a random variable with standard normal distribution and its cosine (non-linear) transformation are not correlated in the sense of Pearson. 

Over the last few years, many non-linear dependence measures have been introduced like the Kernel Canonical Correlation Analysis (KCCA)~\cite{hardoon2009convergence}, the Distance or Brownian Correlation (dCor)~\cite{szekely2009brownian},  the Hilbert-Schmidt Independence Criterion (HSIC and CHSIC)~\cite{Gretton:2005:KMM:1046920.1194914,poczos2012copula} or the Hirschfeld-Gebelein-R\'enyi (HGR)~\cite{renyi1959measures}. 
Comparing those non-linear dependence measures~\cite{lopez2013randomized}, the HGR coefficient seems to be an interesting choice: it is a normalized measure which is capable of correctly measuring linear and non-linear relationships, it can handle multi-dimensional random variables and it is invariant with respect to changes in marginal distributions.
\footnotetext[1]{$\rho(U, V)$ := $\frac{Cov(U;V)}{\sigma_{U}\sigma_{V}}$, where $Cov(U;V)$, $\sigma_{U}$ and $\sigma_{V}$ are the covariance between $U$ and $V$, the standard deviation of $U$ and the standard deviation of $V$, respectively.}


 
\begin{definition}
For two jointly distributed random variables $U \in \mathcal{U}$ and $V \in \mathcal{V}$
, the Hirschfeld-Gebelein-R\'enyi maximal correlation is
defined as:
\begin{align}
HGR(U, V) &= \sup_{\substack{ f:\mathcal{U}\rightarrow \mathbb{R},g:\mathcal{V}\rightarrow \mathbb{R}}} \rho(f(U), g(V)) = \sup_{\substack{ f:\mathcal{U}\rightarrow \mathbb{R},g:\mathcal{V}\rightarrow \mathbb{R}\\
           E(f(U))=E(g(V))=0 \\   E(f^2(U))=E(g^2(V))=1}} E(f(U)g(V))
\label{hgr}
\end{align}
where $\rho$ is the Pearson linear correlation coefficient~\footnotemark[1] with some measurable functions $f$ and $g$ with positive and finite variance. 
\end{definition}
The HGR coefficient is equal to 0 if the two random variables are independent. If they are strictly dependent the value is 1. 
The spaces for the functions $f$ and $g$ are  infinite-dimensional. This property is the reason why the HGR coefficient proved difficult to compute. 

Several approaches rely on Witsenhausen's linear algebra characterization ~\cite{witsenhausen1975sequences} to compute the HGR coefficient. For discrete features, this characterization can be combined with Monte-Carlo estimation of probabilities ~\cite{baharlouei2019r}, or with kernel density estimation (KDE) \cite{mary2019fairness_full} to compute the HGR coefficient. 
We will refer to this second metric, in our experiments, as HGR\_KDE. Note that this metric can be extended to the continuous case by discretizing the density computation.
Another way to approximate this coefficient, Randomized Dependence Coefficient (RDC) \cite{lopez2013randomized}, is to require that $f$ and $g$ belong to reproducing kernel Hilbert spaces (RKHS) and take the largest canonical correlation between two sets of copula random projections. We will make use of this approximated metric 
as HGR\_RDC. Recently a new approach \cite{grari2019fairness} proposes to estimate the HGR by deep neural network. The main idea is to use two inter-connected neural networks to approximate the optimal transformation functions $f$ and $g$ from \ref{hgr}. The $ HGR_{\Theta}(U,V)$ estimator is computed by considering the expectation of the products of standardized outputs of both networks ($\hat{f}_{w_{f}}$ and $\hat{g}_{w_{g}}$). The respective parameters $w_{f}$ and $w_{g}$ are updated by gradient ascent on the objective function to maximize: $J(w_f,w_g)=E[\hat{f}_{w_{f}}(U)\hat{g}_{w_{g}}(V)]$. This estimation has the advantage of being estimated by backpropagation, the same authors therefore present a bias mitigation via a min-max game with an adversarial neural network architecture. However, this attenuation is performed on the predictor output only. Several recent papers \cite{adel2019one, ragonesi2020learning} have shown that performing the attenuation on a representation 
tends to give better results in terms of prediction accuracy while remaining fair in complex real-world scenarios. In this work, we are interested in learning fair representations via this Renyi estimator. 


\section{Theoretical Properties}
\label{sec:Theo_properties}
In this section we study the consistency of the HGR\_NN estimator (referred to as $\widehat{HGR(U,V)}_{n}$), 
and provide a theoretical comparison with 
simple adversarial algorithms that rely on an adversary which predicts the sensitive attribute \cite{zhang2018mitigating,adel2019one}. All the proofs can be found in the Supplementary Material.

\subsection{Consistency of the HGR\_NN}

\theoremstyle{definition}
\begin{definition}{(Strong consistency)} The estimator $\widehat{HGR(U,V)}_{n}$ is strongly consistent if for all $\epsilon > 0$, there exists a positive integer $N$ and a choice of statistics network such that:
\begin{equation}
\forall n \geq N, \abs{HGR(U,V)-\widehat{HGR(U,V)}_{n}} \leq \epsilon, a.s.
\end{equation} 
\end{definition}

As explained in MINE \cite{belghazi2018mutual}, the question of consistency is divided into two problems: a deterministic approximation problem related to the choice of the statistics network, and an estimation problem related to the use of empirical measures. \

The first lemma addresses the approximation problem using universal approximation theorems for neural networks \cite{hornik1989multilayer}:

\begin{lemma}{(approximation)}
Let $\eta > 0$. There exists a family of continuous neural networks $F_{\Theta}$ parametrized by a compact domain $\Theta \subset \mathbb{R}^{k}$, such that
\begin{equation}
 \abs{HGR(U,V)-HGR_{\Theta}(U,V)} \leq \eta.
\end{equation}
\end{lemma}

The second lemma addresses the estimation problem, making use of classical consistency theorems for extremum estimators \cite{geer2000empirical}. It states the almost sure convergence of HGR\_NN to the associated theoretical neural HGR measure as the number of samples goes to infinity: 

\begin{lemma}{(estimation)} 
Let $\eta > 0$, and $F_{\Theta}$  a family of continuous neural networks  parametrized by a compact domain $\Theta \subset \mathbb{R}^{k}$. There exists an $N \in \mathbb{N}$ such that:
\begin{equation}
    \forall n \geq N,  
    \abs{\widehat{HGR(U,V)_n}-HGR_\Theta(U,V)} \leq \eta, a.s.
\end{equation}
\end{lemma}
It is implied here that, from rank $N$, all sample variances are positive in the definition of $\widehat{HGR(U,V)_n}$, which makes the latter well-defined. \

We deduce from these two lemmas the following result:
\begin{theorem}
$\widehat{HGR(U,V)}_{n}$ is strongly consistent.
\end{theorem}
\subsection{Comparison with simple adversarial algorithms}
Given $X$ and $Y$ two one-dimensional random variables, we consider the regression problem:
\begin{align}
    \inf_{f:\mathbb{R}\rightarrow \mathbb{R}} E((Y - f(X))^2)
\end{align}
The variable that minimizes the quadratic risk is $E(Y|X)$. We consider the maximization problem $\sup_{f:\mathbb{R}\rightarrow \mathbb{R}} \rho(f(X),Y)$, which corresponds to the situation where the neural network $g$ is linear in the HGR neural estimator. We have the following result: 
\begin{theorem}
If $E(Y|X)$ is constant, then $\sup_{f} \rho(f(X),Y) = 0$. Else,  $f^{*} \in \argmax_{f} \rho(f(X),Y) $ iff there exists $a,b \in \mathbb{R} $, with $a > 0$, such that:
 \begin{align}
     f^{*}(X) = a E(Y|X) + b
 \end{align}
\end{theorem}
The simpler version of the HGR\_NN, with $g$ linear, finds the optimal function in terms of regression risk, up to a linear transformation that can be found by simple linear regression. The simplified HGR estimation module therefore captures the exact same non-linear dependencies as the predictive adversary in related work \cite{adel2019one,zhang2018mitigating}. 
Thanks to the function $g$, in cases where $Y$ cannot be expressed as a function of $X$ only, the HGR neural network can capture more dependencies than a predictive NN (or equivalently a simplified HGR neural network). Let us consider the following example:
\begin{align}
    \centering
    Y \sim \mathcal{N}(\mu, \sigma^{2}) \hspace{1 cm} X = \arctan(Y^{2}) + U\pi 
\end{align}
where  $U \perp Y$  and $U$ follows a Bernoulli distribution with  $p = \frac{1}{2}$. In this setting, we have $Y^{2} = \tan(X)$, $HGR(X,Y)=1$ and due  to the hidden variable $U$, neither $X$ nor $Y$ can be expressed as a function of the other. In that case, the simplified maximal correlation, $\rho(E(Y|X),Y)$, has the following bounds, with $\alpha = \frac{\mu}{\sigma}$:
    $\sqrt{1-e^{-\frac{\alpha^2}{2}}}\leq \rho(E(Y|X),Y) \leq \sqrt{1-e^{-\frac{\alpha^2}{2}}(1+\alpha^2)^{-\frac{3}{2}}}$.
In the degenerate case $\alpha=0$, we have $E(Y|X)=0$: the predictive neural network cannot find any dependence. For non-zero values of $\alpha$, the distribution of $Y$ is no longer centered around the axis of symmetry of the square function, so that the prediction becomes possible. However, as shown in the inequality above, 
the simplified maximal correlation is less than 1, and close to 0 when $\mu \ll \sigma$. Therefore, as shown by the example, the bilateral approach of the HGR, as opposed to the unilateral approach of predictive models, can capture more dependencies in complex regression scenarios.

In adversarial bias mitigation settings, predictive adversaries might not be able to properly detect bias, which makes it harder for the adversarial algorithm to mitigate it. In fact, prediction retreatment algorithms with predictive adversaries \cite{zhang2018mitigating} achieve the global fairness optimum when $E(S \vert \widehat{Y})=E(S)$
, which does not generally imply demographic parity when $S$ is continuous. On the other hand, adversarial approaches based on the HGR\_NN \cite{grari2019fairness} 
achieve the optimum when $HGR(\widehat{Y},S)=0$, which is equivalent to demographic parity: $P(\widehat{Y}\vert S)=P(\widehat{Y})$.

\section{Method} 
\label{sec:HGR_Adversarial_Network}

\begin{figure}[h]
  \centering
  \includegraphics[width=10.55cm]{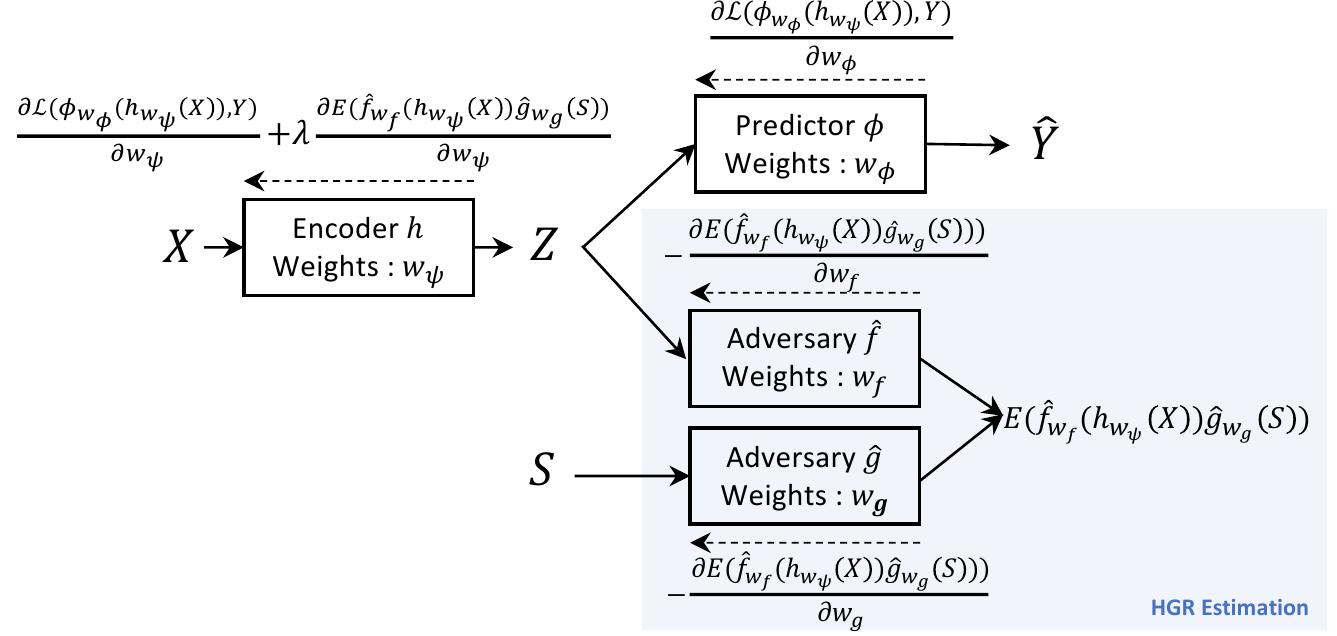}
  \caption{Learning Unbiased Representations via Rényi Minimization}
  \label{fig:hgr_rep_fair}
\end{figure}

The objective is to find a latent representation $Z$ which both minimizes the deviation between the target $Y$ and the output prediction $\widehat{Y}$, 
provided by a function $\phi(Z)$, and does not imply too much dependence with the sensitive $S$. As explained above in section \ref{sec:Problem_statement}, the HGR estimation by deep neural network \cite{grari2019fairness} is a good candidate for standing as the adversary $HGR(Z,S)$ to plug in the global objective (\ref{genericfunction}). Notice, we can consider the latent representation $Z$ or even the sensitive attribute $S$ as multi-dimensional. This can therefore provide a rich representation of the latent space or even take into account several sensitive features at the same time (for e.g. gender and age or the 3 channels of an image see \ref{sec:SyntheticScenario}). The HGR estimation paper \cite{grari2019fairness} considers only the one-dimensional cases for both $U$ and $V$ but we can generalize to the multidimensional cases.

The mitigation procedure follows the optimization problem: 
\begin{eqnarray}
\begin{aligned}
    \argmin_{w_{\phi},w_{\psi} }\max_{{w_f,w_g}} & \ \mathcal{L}(\phi_{\omega_\psi}(h_{\omega_\phi}(X)),Y)+
    \lambda E(\widehat{f}_{w_f}(h_{\omega_\psi}(X))*\widehat{g}_{w_g}(S))
    \label{genericfunction}
\end{aligned}
\end{eqnarray}
where  $\mathcal{L}$ is the predictor loss function 
between the output prediction $\phi_{\omega_\psi}(h_{\omega_\phi}(X)) \in \mathbb{R}$ and the corresponding target $Y$, with $\phi_{\omega_\phi}$ the predictor neural network with parameters $\omega_\phi$ and $Z = h_{\omega_\psi}(X)$ the latent fair representation with $h_{\omega_\psi}$ the encoder neural network, with parameters $\omega_\psi$. The second term, which corresponds to the expectation of the products of standardized outputs of both networks ($\hat{f}_{w_{f}}$ and $\hat{g}_{w_{g}}$), represents the HGR estimation between the latent variable $Z$  and the sensitive attribute $S$. The hyperparameter $\lambda$ controls the impact of the correlation loss in the optimization.

 Figure \ref{fig:hgr_rep_fair} gives the full architecture of our adversarial learning algorithm using the neural HGR estimator between the latent variable and the sensitive attribute. It depicts the encoder function $h_{w_{\psi}}$, which outputs a latent variable $Z$ from $X$, the two neural networks $f_{w_f}$ and $g_{w_g}$, which seek at defining the most strongly correlated transformations of $Z$ and $S$ and the neural network $\phi_{\omega_\phi}$ \vspace{-0.08cm} which outputs the prediction $\widehat{Y}$ from the latent variable $Z$. Left arrows represent gradients back-propagation. The learning is done via stochastic gradient, alternating steps of adversarial maximization and global loss minimization. 
The algorithm (more details in the supplementary) takes as input a training set from which it samples batches of size $b$ at each iteration. At each iteration it first standardizes the output scores of networks $f_{w_f}$ and $g_{w_g}$ to ensure 0 mean and a variance of 1 on the batch. Then it computes 
the HGR neural estimate and the prediction loss for the batch. 
At the end of each iteration, the algorithm updates the parameters of the adversary $w_f$ and $w_g$ by one step of gradient ascent and the prediction parameters $\omega_\phi$ as well as the encoder parameters $\omega_\psi$ by one step of gradient descent. 
\section{Experiments}
\label{sec:results}
\subsection{Synthetic Scenario}
\label{sec:SyntheticScenario}

Inspired by \cite{louppe2017learning}, we consider the following toy scenario in a binary target and continuous standard gaussian sensitive attribute setting: 
\begin{subequations}
    \begin{align}
        \centering
        X|S=s \sim \mathcal{N}
        \begin{bmatrix}
        \begin{pmatrix}
        0\\
        0
        \end{pmatrix}\!\!,&
        \begin{pmatrix}
        1 & -\frac{1}{2} \\
        -\frac{1}{2} & 1 
        \end{pmatrix}
        \end{bmatrix}
        \hspace{1 cm}
        &\text{when  } Y = 0, \\
        X|S=s \sim \mathcal{N}
        \begin{bmatrix}
        \begin{pmatrix}
        1\\
        1+3\sin{s}
        \end{pmatrix}\!\!,&
        \begin{pmatrix}
        1 & 0 \\
        0 & 1 
        \end{pmatrix}
        \end{bmatrix}
        \hspace{1 cm}
        &\text{when  } Y = 1
        \label{toy_sin}
        \centering
        \end{align}
\end{subequations}

\begin{figure}[H]
\centerfloat
\subfloat[Biased model: $\lambda=0$ ;  $HGR(Z,S)=52\%$ ; $HGR(\widehat{Y},S)=30\%$ ; $Acc=79\%$]{\includegraphics[scale=0.25]{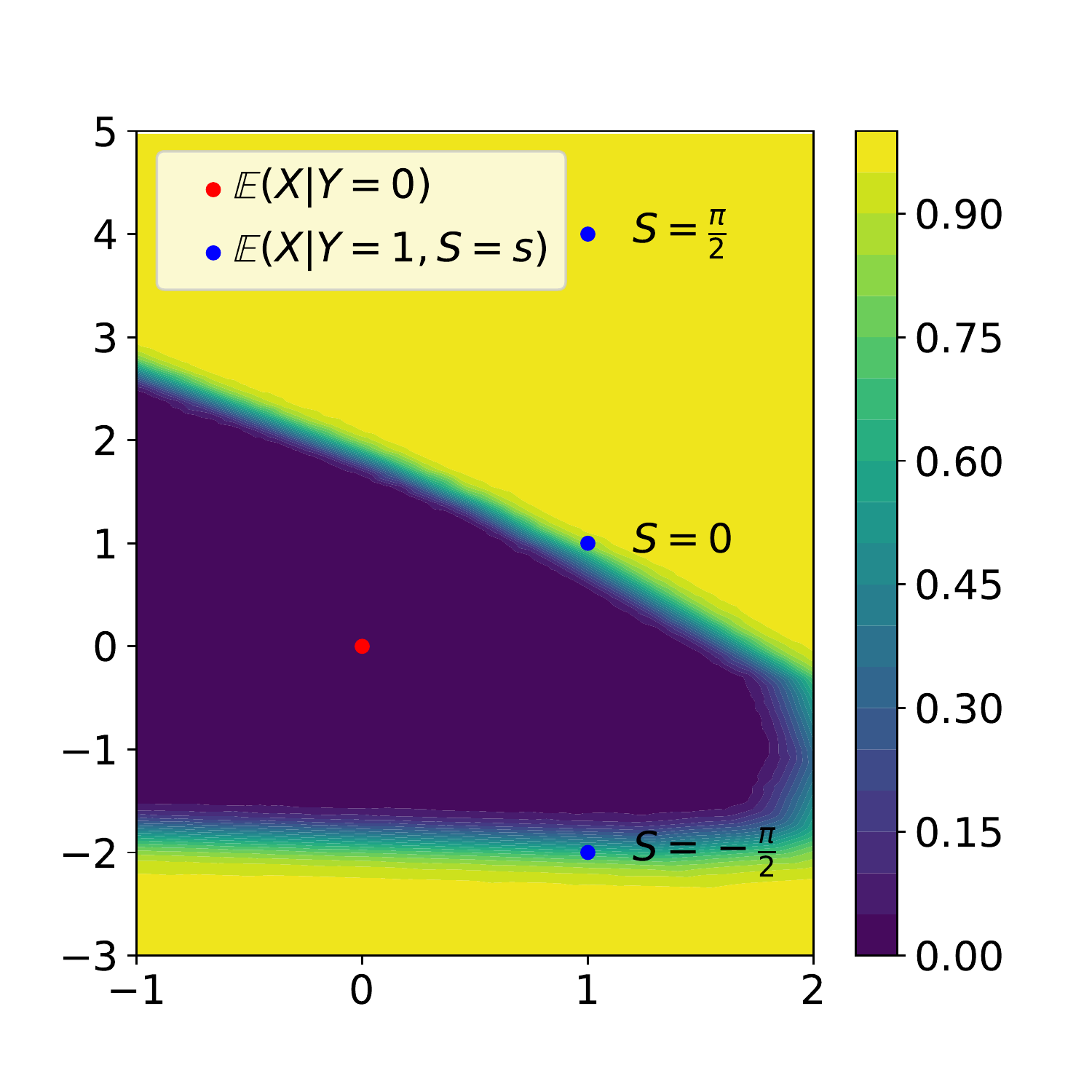} \\
\includegraphics[scale=0.25]{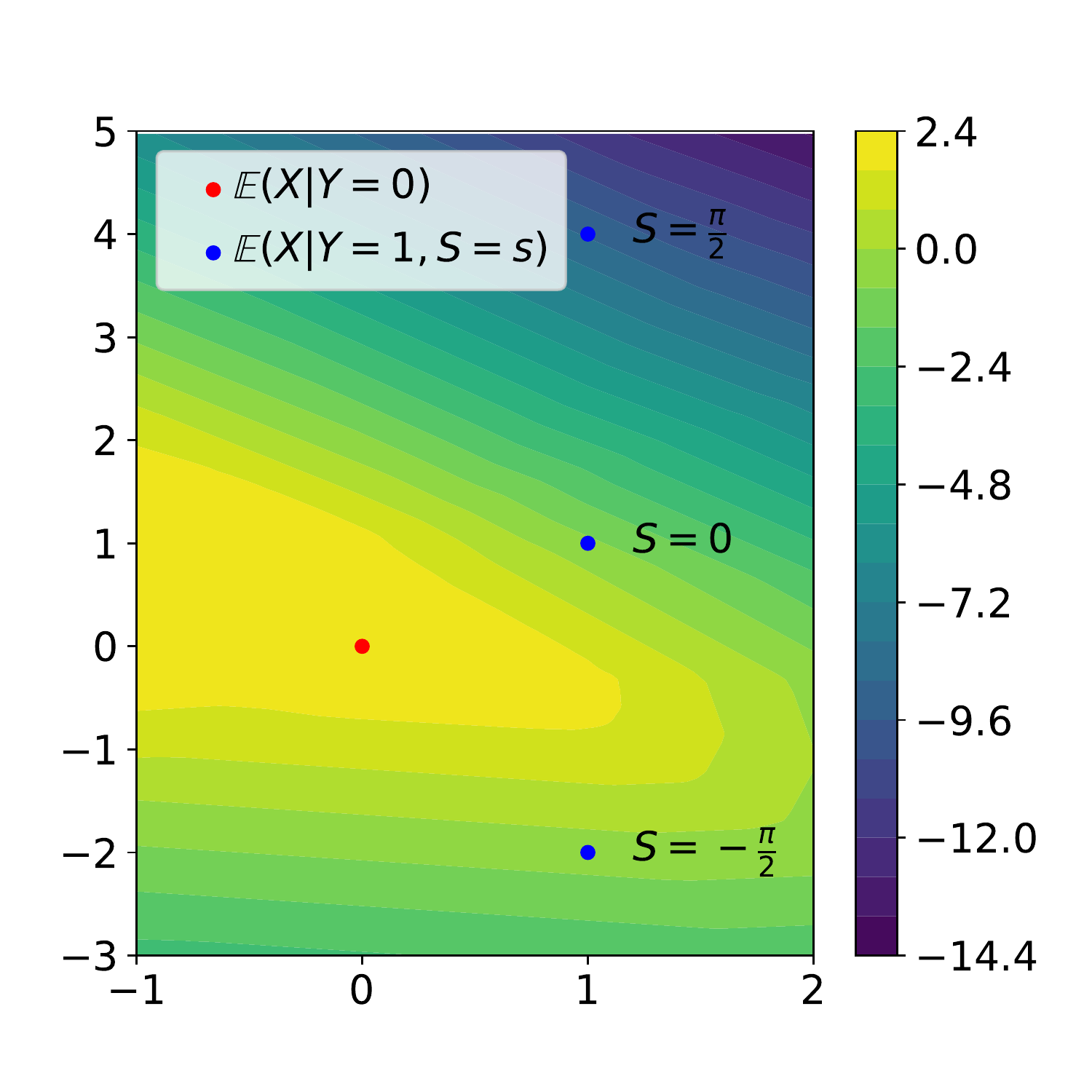} \\
\includegraphics[scale=0.25]{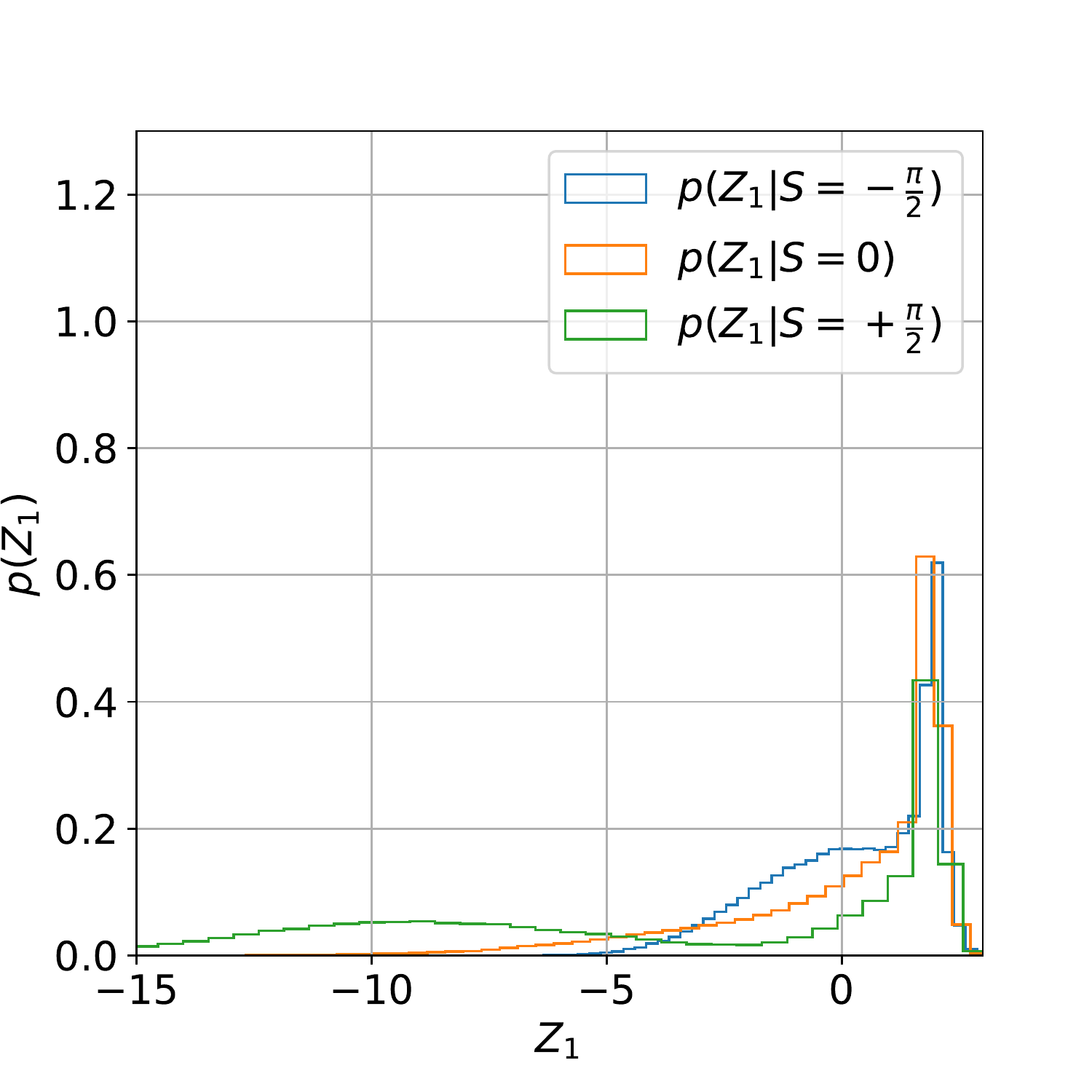} \\
\includegraphics[scale=0.25]{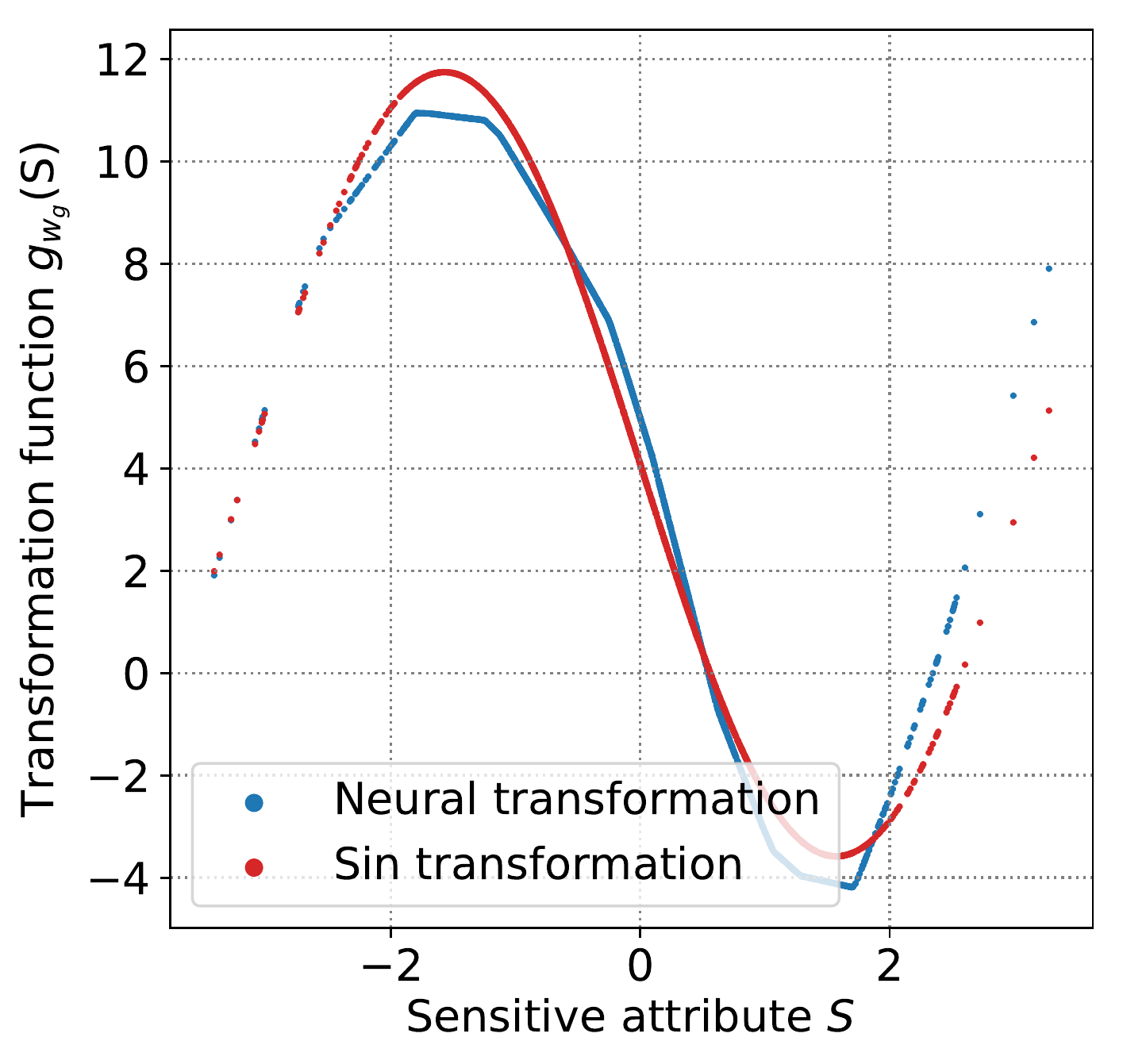}
}  \\
\subfloat[Unbiased model: $\lambda=13$ ;  $HGR(Z,S)=5\%$ ; $HGR(\widehat{Y},S)=4\%$ ; $Acc=68\%$]{\includegraphics[scale=0.25]{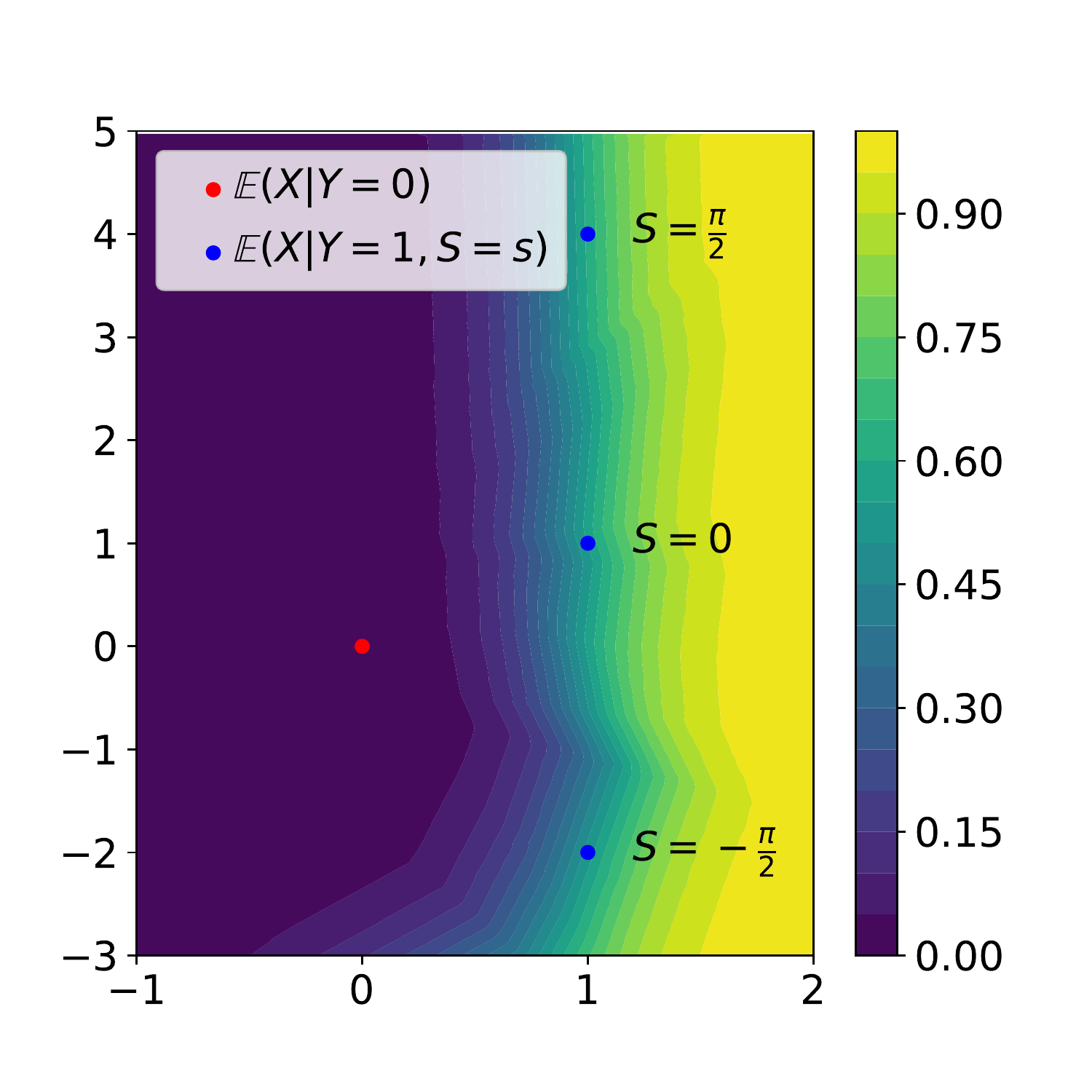} \\
\includegraphics[scale=0.25]{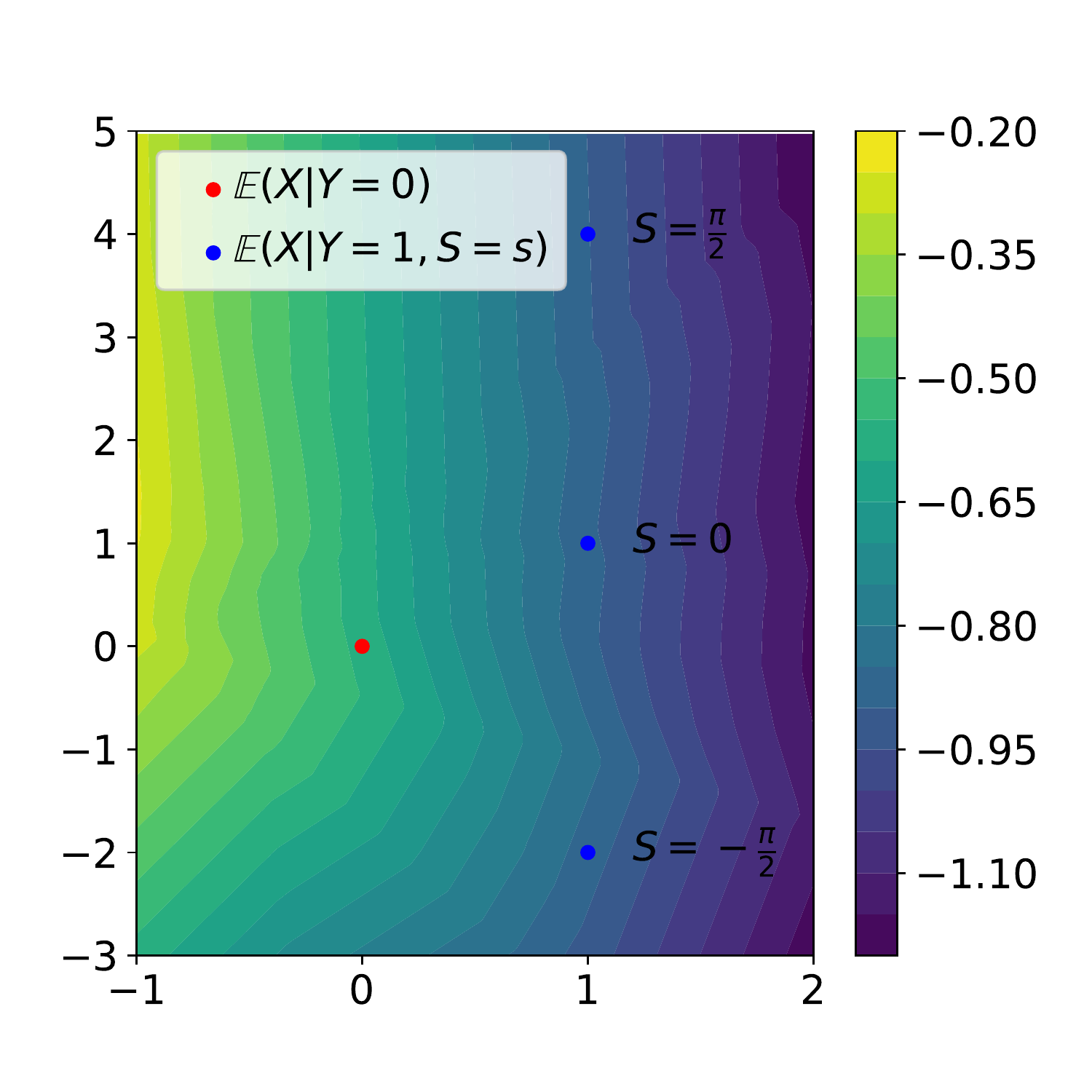} \\
\includegraphics[scale=0.25]{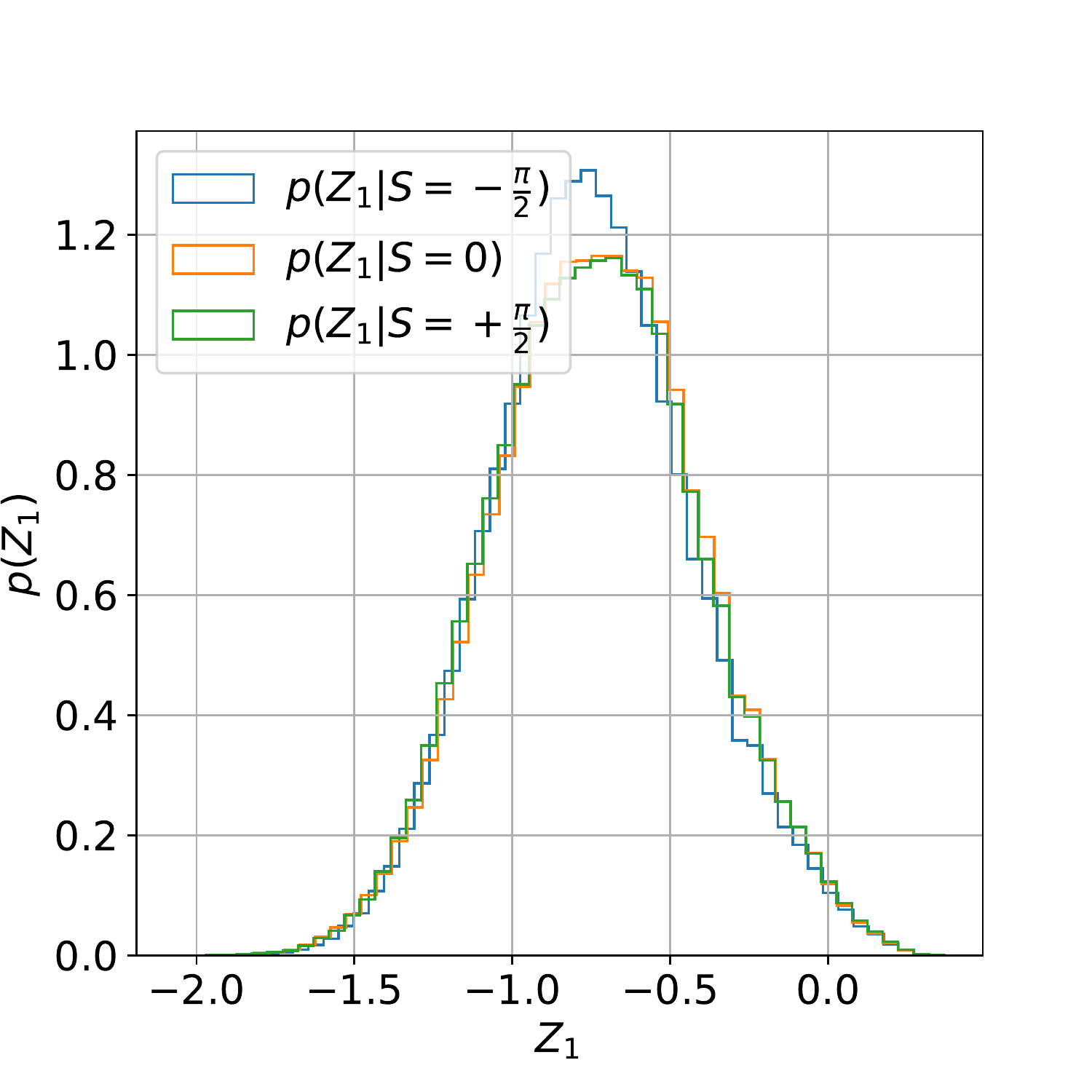} \\
\includegraphics[scale=0.25]{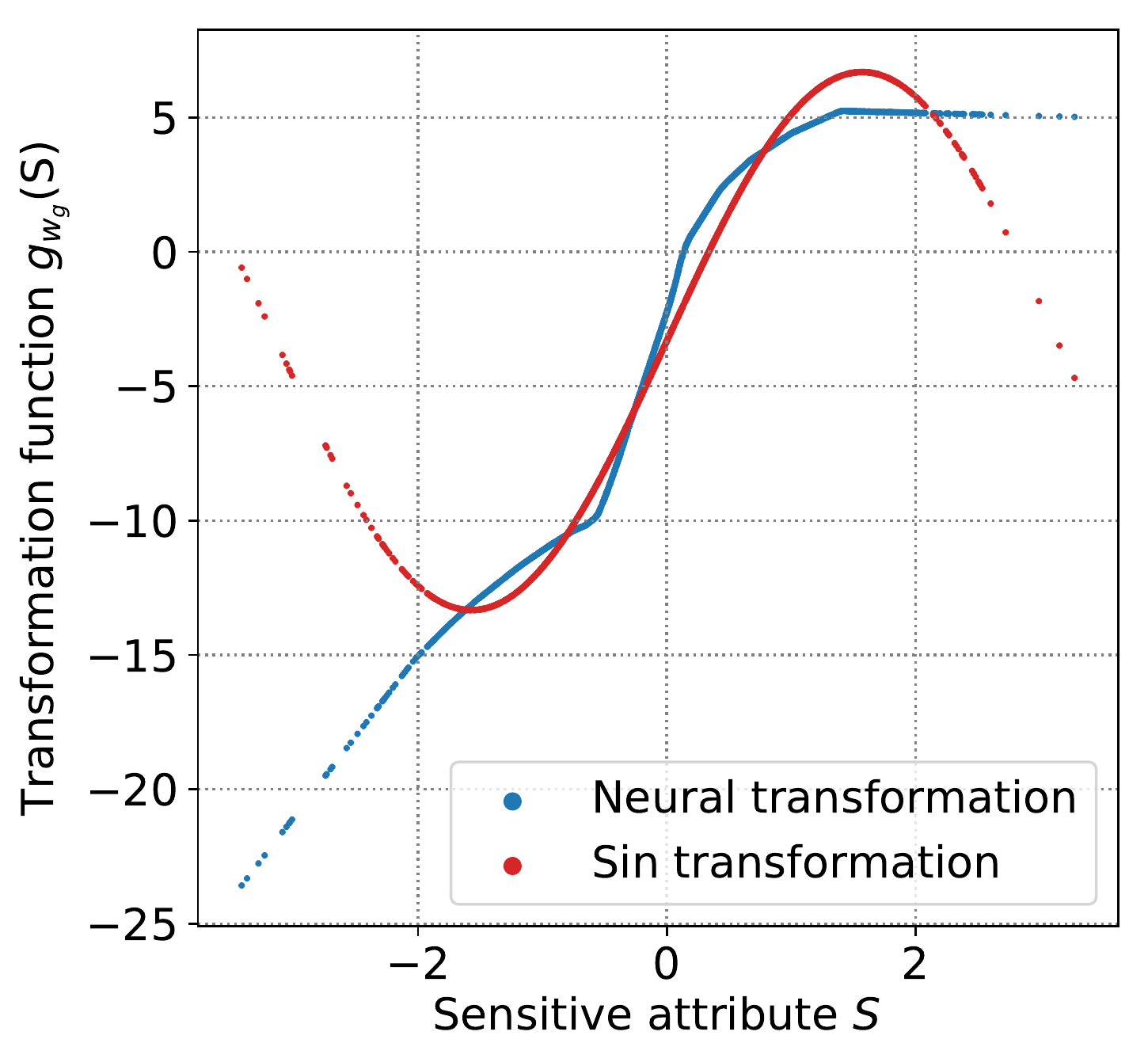}}

\caption{Toy example. (Left) Decision surface in the ($X_1$, $X_2$) plane. The figure (a) shows the decision surface for a biased model focused on a prediction loss. $\widehat{Y}$ values are highly correlated with $S$, samples with $S$ around $\frac{\pi}{2}$ and $Y=1$ being easier to classify than those with $S$ between $-\frac{\pi}{2}$ and $0$. The figure (b) shows decision surfaces for our fair model. \vspace{-0.1cm} These are vertical, meaning that only $X_1$  influences the classification, and therefore $\widehat{Y}$ is no longer biased w.r.t $S$. (Middle left) $Z_1$-slices in the ($X_1$, $X_2$) plane. The comparison between the figure below and above highlights the fact that adversarial training allows to create an unbiased representation $Z$. (Middle right) Conditional probability densities of $Z_1$ at $S=-\frac{\pi}{2}, 0, \frac{\pi}{2}$. With $\lambda = 0$, the densities are dependent on $S$, whereas they are not anymore with adversarial training. (Right) In blue, the function modeled by the neural network $g$ in the HGR Neural Network. In red, the closest linear transformation of $\sin(S)$ to $g(S)$.}
\label{figure_toy}
\end{figure}
Our goal is to learn a representation $Z$ of the input data that is no longer biased w.r.t $S$, while still accurately predicting the target value $Y$. Figure \ref{figure_toy}  compares the results of both a biased model (a) with a hyperparameter $\lambda = 0$ and an unbiased model (b) with $\lambda = 13$ applied on the toy scenario data. In the context of the Rényi Minimization method, it is interesting to observe the maximal correlation functions learnt by the adversary. When $\lambda = 0$, the adversary with sensitive attribute input models the $\sin$ function up to a linear transformation, which also maximizes the correlation with the input data as shown in \eqref{toy_sin}. In that case, the representation $Z$ still carries the bias of $X$ w.r.t $S$, in the same $\sin$ shape. When $\lambda = 13$, the neural network $g$ is unable to find the $\sin$ function, which seems to indicate that the representation $Z$ does not carry the bias w.r.t $S$ anymore. This is confirmed by the low HGR coefficient between $Z$ and $S$, the $Z_1$-slices as well as the conditional densities of $Z_1$ at different values of $S$. Not only does the adversarial induces  an unbiased representation, it also leads to an almost completely unbiased target $\widehat{Y}$, as shown by the vertical decision surfaces and the $4 \%$ HGR 
between $\widehat{Y}$ and $S$. 
This at the cost of of a slight loss of accuracy, with an $11 \%$ decrease.  


\subsection{MNIST with Continuous Color Intensity}

Before considering real-world experiments, we follow the MNIST experimental setup defined by Kim et al. \cite{kim2019learning}, which considers a digit classification task with a color bias planted into the MNIST data set \cite{lecun2010mnist, coloredMNIST}. In the training set, 
ten distinct colors are assigned to each class. More precisely, for a given training image, a color is sampled from the isotropic normal distribution with the corresponding class mean color, and a variance parameter $\sigma^2$. For a given test image, a mean color is randomly chosen from one of the ten mean colors, without considering the test label, and a color is sampled from the corresponding normal distribution (with variance $\sigma^2$). Seven transformations of the data set are designed with this protocol, with seven values of $\sigma^2$ equally spaced between 0.02 and 0.05. A lower value of $\sigma^2$ implies a higher color bias in the training set, making the classification task on the testing set more difficult, since the model can base its predictions on colors rather than shape.
The sensitive feature, color, is 
encoded as 
a vector with 3 continuous coordinates. 
For each algorithm and for each data set, we obtain the best hyperparameters by grid search in five-fold cross validation. 

Results, 
in terms of accuracy, can be found in Table \ref{tab:continuous_case}. Notice, the state-of-the-art obtains different results than reported ones because we consider a continuous sensitive feature and not a 24-bit binary encoding. Our adversarial algorithm achieves the best accuracy on the test set for the seven scenarios. 
The most important gap is for the smallest sigma where the generalisation is the most difficult. The larger number of degrees of freedom carried by the two functions $f$ and $g$ made it possible to capture more unbiased information than the other algorithms on the multidimensional  
variables $Z$ and $S$.

\begin{table}[h]
\label{tab:continuous_case}
\centering
\resizebox{\textwidth}{!}{
\begin{tabular}{@{}rccccccc@{}}
\toprule
\multicolumn{1}{c}{} & \multicolumn{7}{c}{Color variance}                    \\ \midrule
Training      & $\sigma$ = 0.020     & $\sigma$= 0.025     & $\sigma$ = 0.030     & $\sigma$ = 0.035     & $\sigma$ = 0.040     & $\sigma$ = 0.045     & $\sigma$ = 0.050     \\ \midrule
ERM ($\lambda$ = 0.0) & 0.476 $\pm$ 0.005 & 0.542 $\pm$ 0.004 & 0.664 $\pm$ 0.001 & 0.720 $\pm$ 0.010 & 0.785 $\pm$ 0.003 & 0.838 $\pm$ 0.002 & 0.870 $\pm$ 0.001 \\
Ragonesi \cite{ragonesi2020learning}  &  0.592 $\pm$ 0.018  & 0.678 $\pm$ 0.015 & 0.737 $\pm$ 0.028 & 0.795 $\pm$ 0.012 & 0.814 $\pm$ 0.019 & 0.837 $\pm$ 0.004 & 0.877 $\pm$ 0.010 \\ 

Zhang et al. \cite{zhang2018mitigating}  &  0.584 $\pm$ 0.034 & 0.625 $\pm$ 0.033  & 0.709 $\pm$ 0.027 & 0.733 $\pm$ 0.020 & 0.807 $\pm$ 0.013 & 0.803 $\pm$ 0.027  & 0.831 $\pm$ 0.027 \\ 
Kim et al.  \cite{kim2019learning}  & 0.645 $\pm$ 0.015 & 0.720 $\pm$ 0.014 & 0.787 $\pm$ 0.018 & 0.827 $\pm$ 0.012 & 0.869 $\pm$ 0.023 & 0.882 $\pm$ 0.019 & 0.900 $\pm$ 0.012 \\ 
Grari et al. \cite{grari2019fairness}  & 0.571 $\pm$ 0.014 & 0.655 $\pm$ 0.022 & 0.721 $\pm$ 0.030 & 0.779 $\pm$ 0.011 & 0.823 $\pm$ 0.013 & 0.833 $\pm$ 0.026 &  0.879 $\pm$ 0.010 \\ 
Adek et al. \cite{grari2019fairness}  & 0.643 $\pm$ 0.014 & 0.655 $\pm$ 0.022 & 0.721 $\pm$ 0.030 & 0.779 $\pm$ 0.011 & 0.823 $\pm$ 0.013 & 0.833 $\pm$ 0.026 &  0.879 $\pm$ 0.010 \\ 
Ours  & \textbf{0.730} $\pm$ 0.008 & \textbf{0.762} $\pm$ 0.021 & \textbf{0.808} $\pm$ 0.011 & \textbf{0.838} $\pm$ 0.010 & \textbf{0.878} $\pm$ 0.011 & \textbf{0.883} $\pm$ 0.012 & \textbf{0.910} $\pm$ 0.007 \\
\bottomrule
\hline
\end{tabular}
}
\caption{MNIST with continuous color intensity}
\end{table}

\subsection{Real-world Experiments}
Our experiments on real-world data are performed on five data sets.
In three data sets, the sensitive and the outcome true value are both continuous: the US Census data set \cite{USCensus}, the Motor data set  \cite{pricinggame15} and the Crime data set \cite{Dua:2019}. On two other data sets, the target is binary 
and the sensitive features are continuous: The COMPAS data set~\cite{angwin2016machine} and the Default data set~\cite{Yeh:2009:CDM:1464526.1465163}. 
For all data sets, we repeat five experiments by randomly sampling two subsets, 80\% for the training set and 20\% for the test set. Finally, we report the average of the mean squared error (MSE), the accuracy (ACC) and the mean of the fairness metrics HGR\_NN \cite{grari2019fairness}, HGR\_KDE \cite{mary2019fairness_full}, HGR\_RDC \cite{lopez2013randomized} and MINE \cite{belghazi2018mutual} on the test set. Since none of these fairness measures are fully reliable (they are only estimations which are used by the compared models), we also use    the $FairQuant$ metric \cite{grari2019fairness},  based on the   
quantization of the test samples 
in 50 quantiles w.r.t. to the sensitive attribute. The metric corresponds to the mean absolute difference between the global average prediction and the mean prediction of each quantile. 

As a baseline, we use a classic, "unfair" deep neural network, Standard NN. 
We compare our approach with state-of-the-art algorithms. 
We also compare 
the Fair MINE NN\cite{grari2019fairness} algorithm where fairness is achieved with the MINE estimation of the mutual information as a penalization in prediction retreatment (lower right in Figure 1 in appendix).
For all the different fair representation algorithms, we assign the latent space with only one hidden layer with 64 units.  
Mean normalization was applied to all the 
outcome true values. Results of our experiments can be found in Table~\ref{tab:results_demographic_parity}. For all of them, we attempted to obtain comparable results by giving similar accuracy to all models, 
via the hyperparameter $\lambda$ (different for each model). For each algorithm and for each data set, we obtain the best hyperparameters by grid search in five-fold cross validation (specific to each of them). 


As expected, the baseline, Standard NN, is the best predictor but also the most biased one. It achieves the lowest prediction errors and ranks amongst the highest and thus worst values for all fairness measures for all data sets and tasks. While being better in terms of 
accuracy, our fair representation algorithm achieves on four data sets (except on the Crime data set) the best level of fairness assessed by HGR estimation, MINE and FairQuant. On the Crime data set, the approach by Mary2019~\cite{mary2019fairness_full}~\footnotemark[2] gets slightly better results but with a very high volatility. Note, Adel \cite{adel2019one} with the fair adversarial representation obtains (except on the Crime data set) better results  than Zhang \cite{zhang2018mitigating} which corresponds to the simple adversarial architecture.



\begin{table*}[!h]
\label{tab:results_demographic_parity}
\centering

\resizebox{\textwidth}{!}{
\begin{tabular}{||l|l||l|l|l|l|l|l||}
\hhline{========} 
 \cline{3-7}
\multicolumn{2}{||c||}{ } & MSE & HGR\_NN & HGR\_KDE & HGR\_RDC & MINE & FairQuant  \\ 
\hline
\multirow{7}{*}{\rotatebox[origin=c]{90}{US Census}}
& Standard NN & 0.274 $\pm$ 0.003 & 0.212 $\pm$ 0.094 & 0.181$\pm$ 0.00 & 0.217 $\pm$ 0.004 & 0.023 $\pm$ 0.018 & 0.059 $\pm$ 0.00 \\
& Grari et al. \cite{grari2019fairness} & 
0.526 $\pm$  0.042 & 0.057 $\pm$  0.011 & 0.046 $\pm$  0.030  & 0.042 $\pm$  0.038  & \textbf{0.001} $\pm$ 0.001 & 0.008 $\pm$ 0.015 \\
& Mary2019~\cite{mary2019fairness_full}  & 0.541  $\pm$ 0.015   & 0.075  $\pm$ 0.013 & 0.061  $\pm$   0.006 & 0.078  $\pm$  0.013   & 0.002 $\pm$ 0.001 &  0.019  $\pm$  0.004  \\
& Fair MINE NN & 0.537 $\pm$ 0.046 & 0.058 $\pm$ 0.042 & 0.048 $\pm$ 0.029 & 0.045 $\pm$ 0.037  & \textbf{0.001} $\pm$ 0.001 & 0.012 $\pm$  0.016 \\
& Adel \cite{adel2019one}  &  0.552  $\pm$ 0.032  &  0.100  $\pm$  0.028 &   0.138 $\pm$ 0.042  &  0.146  $\pm$ 0.031   & 0.003 $\pm$ 0.003 &  0.035  $\pm$  0.011 \\

& Zhang et al. \cite{zhang2018mitigating}   &  0.727 $\pm$ 0.264  &  0.097 $\pm$ 0.038 &  0.135  $\pm$ 0.036  &  0.165 $\pm$ 0.028   & 0.009 $\pm$ 0.005 &  0.022  $\pm$  0.019    \\  

& Madras et al.    \cite{pmlr-v80-madras18a}  &  0.512 $\pm$ 0.033  &  0.129  $\pm$  0.010  & 0.158  $\pm$ 0.009  &  0.173 $\pm$ 0.012  & 0.007 $\pm$  0.007  & 0.041 $\pm$ 0.003 \\ 

& Sadeghi et al. \cite{sadeghi2019global} &  0.526 $\pm$ 0.006  & 0.077 $\pm$ 0.031 &   0.136 $\pm$ 	0.001 & 0.146 $\pm$ 0.001 &   0.008 $\pm$ 0.003  &  0.035 $\pm$ 0.000  \\ 
& Ours  &  \textbf{0.523} $\pm$ 0.035  &  \textbf{0.054} $\pm$ 0.015  &  \textbf{0.044} $\pm$ 0.032  & \textbf{0.041}  $\pm$  0.031  & \textbf{0.001} $\pm$ 0.001 &  \textbf{0.007}  $\pm$ 0.002   \\ 
\hline
\multirow{7}{*}{\rotatebox[origin=c]{90}{Motor}}
& Standard NN & 0.945  $\pm$ 0.011   & 0.201  $\pm$ 0.094 & 0.175  $\pm$ 0.0 & 0.200  $\pm$ 0.034  & 0.188 $\pm$ 0.005 & 0.008   $\pm$ 0.011  \\ 
& Grari et al. \cite{grari2019fairness} & 0.971   $\pm$ 0.004  & 0.072 $\pm$ 0.029 & 0.058 $\pm$ 0.052 &  \textbf{0.066} $\pm$ 0.009  & \textbf{0.000} $\pm$ 0.000 & 0.006  $\pm$  0.02 \\ 
& Mary2019~\cite{mary2019fairness_full}  & 0.979 $\pm$ 0.119   & 0.077 $\pm$  0.023 & 0.059  $\pm$  0.014 & 0.067  $\pm$  0.028  & 0.001 $\pm$ 0.001 & 0.006  $\pm$   0.002   \\ 
& Fair MINE NN & 0.982  $\pm$ 0.003 & 0.078   $\pm$ 0.013 & 0.068 $\pm$ 0.004 & 0.069 $\pm$ 0.009  & \textbf{0.000} $\pm$ 0.000 & \textbf{0.004} $\pm$ 0.001  \\

& Adel \cite{adel2019one} &   0.979 $\pm$  0.003 & 0.101 $\pm$ 0.04 &  0.09  $\pm$  0.03   &  0.101  $\pm$  0.04  & 0.002 $\pm$ 0.002 &  0.009  $\pm$  0.004    \\

& Zhang et al. \cite{zhang2018mitigating}   & 0.998  $\pm$ 0.004  & 0.076 $\pm$ 0.034 &  0.091 $\pm$ 0.024  &  0.129 $\pm$  0.08   & 0.001 $\pm$ 0.001 &  \textbf{0.004} $\pm$  0.001    \\ 
& Madras et al.    \cite{pmlr-v80-madras18a}  & 0.978  $\pm$ 0.004   &  0.096  $\pm$ 0.035  & 0.083  $\pm$  0.020 &   0.099 $\pm$ 0.030 &  0.004 $\pm$ 0.002  & 0.008 $\pm$  0.001   \\ 
						
& Sadeghi et al. \cite{sadeghi2019global} &  0.975 $\pm$ 0.017 &  0.102 $\pm$ 0.020 & 0.115  $\pm$ 0.027 & 0.129 $\pm$ 0.039 & 0.001  $\pm$ 0.001  &	0.001  $\pm$  0.001 \\ 							

& Ours  &   \textbf{0.962} $\pm$ 0.002  & \textbf{0.070}  $\pm$ 0.011  &  \textbf{0.055} $\pm$  0.005  &  0.067 $\pm$   0.006  & \textbf{0.000} $\pm$ 0.000 & \textbf{0.004}  $\pm$  0.001     \\ 
\hline
\multirow{7}{*}{\rotatebox[origin=c]{90}{Crime}}
& Standard NN                        &    0.384   $\pm$ 0.012   &  0.732  $\pm$ 0.013 &  0.525   $\pm$ 0.013  &  0.731    $\pm$ 0.009  & 0.315 $\pm$ 0.021 & 0.353   $\pm$ 0.006  \\ 
& Grari et al. \cite{grari2019fairness}    &     0.781 $\pm$  0.016    &  \textbf{0.356} $\pm$ 0.063 &  0.097  $\pm$ 0.022  &  \textbf{0.171} $\pm$ 0.03   & \textbf{0.009} $\pm$ 0.008 & \textbf{0.039}$\pm$ 0.008  \\ 
& Mary2019~\cite{mary2019fairness_full}    &   \textbf{0.778} $\pm$ 0.103  &   0.371 $\pm$  0.116   &  0.115 $\pm$ 0.046 &  0.177  $\pm$ 0.054   & 0.024 $\pm$ 0.015 &      0.064 $\pm$ 0.023 \\ 
& Fair MINE NN  &   0.782 $\pm$  0.034  &  0.395 $\pm$ 0.097 &  0.110 $\pm$  0.022 &  0.201 $\pm$ 0.021   & 0.032 $\pm$ 0.029 &  0.136 $\pm$ 0.012   \\ 
& Adel \cite{adel2019one}    &   0.836 $\pm$ 0.005  &  0.384 $\pm$ 0.037 &  0.170 $\pm$ 0.027  &  0.371 $\pm$ 0.035   & 0.058 $\pm$ 0.027 &  0.057  $\pm$ 0.007    \\

& Zhang et al.    \cite{zhang2018mitigating}  &  0.787 $\pm$ 0.134  & 0.377 $\pm$ 0.085 &  0.153 $\pm$ 0.056 &  0.313 $\pm$ 0.087   & 0.037 $\pm$ 0.022 & 0.063 $\pm$   0.046   \\ 

& Madras et al.    \cite{pmlr-v80-madras18a}  &  0.725  $\pm$  0.023 &   0.312 $\pm$ 0.022  &  0.290 $\pm$ 0.027 &    0.175 $\pm$ 0.016 &  0.036 $\pm$  0.013  & 0.103  $\pm$ 0.015  \\ 

& Sadeghi et al. \cite{sadeghi2019global} &  0.782 $\pm$ 	0.002 & 0.474  $\pm$  0.006 &   0.123 $\pm$ 0.000 & 0.315 $\pm$ 0,009 &   0.098 $\pm$  0.035  & 0.062  $\pm$ 0.001  \\ 

& Ours &  0.783 $\pm$ 0.031 \hspace*{1px} & 0.369 $\pm$ 0.074 &  \textbf{0.087} $\pm$ 0.031  &  0.173 $\pm$  0.044 & 0.011 $\pm$ 0.006 &  0.043 $\pm$  0.012    \\ 
\end{tabular}
}
\resizebox{\textwidth}{!}{
\begin{tabular}{||l|l||l|l|l|l|l|l||l||}

\hhline{--------} 
\cline{3-8}
\multicolumn{2}{||c||}{ } & ACC & HGR\_NN & HGR\_KDE & HGR\_RDC & MINE  & FairQuant  \\ 
\hline
\multirow{7}{*}{\rotatebox[origin=c]{90}{COMPAS}}
& Standard NN &  68.7\% $\pm$ 0.243  & 0.363 $\pm$ 0.005 & 0.326  $\pm$ 0.003  &   0.325 $\pm$ 0.008  & 0.046 $\pm$ 0.028 &  0.140 $\pm$ 0.001 \\ 

& Grari et al. \cite{grari2019fairness} &  59.7\% $\pm$ 2.943  & 0.147 $\pm$ 0.000 & 0.121  $\pm$  0.002 &   0.101 $\pm$ 0.007 & 0.004 $\pm$ 0.001  &  0.018 $\pm$ 0.018    \\ 
& Fair MINE NN & 54.4\% $\pm$  7.921 &  0.134 $\pm$ 0.145 &  0.123  $\pm$ 0.111   & 0.141 $\pm$ 0.098 & 0.014  $\pm$  0.023 & 0.038  $\pm$  0.050   \\ 
& Adel \cite{adel2019one}&  55.4\%  $\pm$ 0.603 & 0.118 $\pm$ 0.022  &  0.091 $\pm$  0.012   &    0.097 $\pm$ 0.034 & 0.006 $\pm$ 0.007 & 0.013 $\pm$ 0.016    \\ 

&  Zhang et al. \cite{zhang2018mitigating} &  51.0\% $\pm$  3.550 & 0.116 $\pm$ 0.000  & 0.081  $\pm$ 0.003  &  0.086 $\pm$ 0.010 & 0.002 $\pm$ 0.003 & \textbf{0.010}  $\pm$ 0.005     \\
& Madras et al.    \cite{pmlr-v80-madras18a}  &  54.9\% $\pm$  2.221 & 0.175   $\pm$ 0.000   &   0.116 $\pm$ 0.015  &  0.107 $\pm$ 0.026 & 0.011 $\pm$ 0.020 &  0.005 $\pm$ 0.003    \\ 

& Sadeghi et al. \cite{sadeghi2019global} &  54.3\% $\pm$ 0.024 & 0.194 $\pm$  0.052 &  0.237 $\pm$ 0.040 & 0.264 $\pm$ 0.054  &   0.003 $\pm$ 0.003   & 0.003  $\pm$   0.003 \\
& Ours &  \textbf{60.2}\% $\pm$ 3.076  & \textbf{0.063} $\pm$ 0.024 & \textbf{0.068}  $\pm$ 0.018   &  \textbf{0.067} $\pm$ 0.014 & \textbf{0.001} $\pm$ 0.002 & 0.011 $\pm$ 0.018   \\ 
\hline
\multirow{6}{*}{\rotatebox[origin=c]{90}{Default}}
 & Standard NN &  82.1\% $\pm$  0.172 & 0.112 $\pm$ 0.013 &   0.067 $\pm$ 0.010    & 0.089  $\pm$ 0.014 & 0.002 $\pm$ 0.001 &  0.015 $\pm$ 0.002    \\

& Grari et al. \cite{grari2019fairness}&  79.9\% $\pm$ 2.100  & 0.082 $\pm$ 0.015  &  0.075 $\pm$ 0.019   &    0.072 $\pm$ 0.010  & 0.001 $\pm$  0.001 &   0.007 $\pm$   0.007  \\ 

& Adel \cite{adel2019one}  &   79.2\% $\pm$ 1.207  & 0.054 $\pm$ 0.025 & 0.048  $\pm$ 0.015  &  0.064  $\pm$ 0.009  & 0.001 $\pm$ 0.001 &  0.005 $\pm$ 0.002    \\

& Fair MINE NN &  80.1\% $\pm$ 2.184   &  0.093 $\pm$ 0.020 &   0.057 $\pm$ 0.002   &  0.066 $\pm$ 0.012  & 0.001 $\pm$ 0.001 &  0.008 $\pm$ 0.001    \\ 


& Zhang \cite{zhang2018mitigating} &   77.9\% $\pm$  9.822 & 0.052 $\pm$ 0.017 &   \textbf{0.044} $\pm$  0.013   &  0.056  $\pm$ 0.004 & \textbf{0.000} $\pm$ 0.000 & \textbf{0.004}  $\pm$ 0.000    \\ 
& Madras et al.    \cite{pmlr-v80-madras18a}  &  78.3\% $\pm$  0.605 &   0.064 $\pm$ 0.025   &  0.052 $\pm$ 0.018  &  0.061 $\pm$ 0.012 & 0.001 $\pm$ 0.001 & 0.003   $\pm$   0.005  \\ 

& Sadeghi et al. \cite{sadeghi2019global} &  79.7\% $\pm$ 0.236 & 0.074 $\pm$ 	0.019 &  0.062 $\pm$  0.013 & 0.098 $\pm$  0.041 &   0.002 $\pm$ 0.002   & 0.003  $\pm$ 0.002  \\ 				
& Ours & \textbf{80.8}\%  $\pm$ 0.286  &  \textbf{0.041} $\pm$ 0.008 & \textbf{0.044}  $\pm$ 0.006   &   \textbf{0.047} $\pm$ 0.002  & 0.001 $\pm$ 0.002 &  0.005 $\pm$ 0.001    \\

\hline

\hhline{========} 
\end{tabular}
}
\centering
\\
\caption{Experimental results - 
Best performance among fair algorithms in bold.}
\end{table*}


\section{Conclusion}
\label{sec:conclusion}
We present a new adversarial learning approach to produce fair representations with a continuous sensitive attribute. We leverage the HGR coefficient, which is efficient in capturing non-linear dependencies, and propose 
to minimize a neural estimation of the HGR correlation between the latent space representation and the sensitive attributes. 
This method proved to be very efficient for different fairness metrics on various artificial and real-world data sets. For further investigation, we will apply this architecture for 
information bottleneck purposes (e.g., for data privacy), which might be improved with an HGR\_NN penalization as suggested in \cite{asoodeh2015maximal}.


\section{Broader Impact}
Machine learning models are playing an increasingly important role in decision making and the impacts can be dramatic. For example, in the case of banks and insurance companies, these algorithms are applied to establish credit approval or fraud detection and more recently, these models have been applied to generate predictions for criminal recidivism with COMPAS in the United States. 
The stakes are therefore major for citizens, and we must understand and master them. 
The standard machine learning models only 
optimize accuracy and are prone to learn all the relevant information for the task whether they are sensitive or not. Many incidents of discrimination on the basis of gender, race or others have been documented these recent years. 
This present work should therefore be positive for the well-being of society, for citizens on the one hand to avoid discrimination and for companies and institutions to allow for better control of their algorithms. More precisely this present work has two significant advantages, as it provides first an assessment of the sensitive bias in the latent representation of the classical deep learning model or in the prediction itself and second it makes its mitigation possible. 
The regulation law for data privacy and fairness in machine learning 
has evolved significantly over the past years in many countries and is likely to be more restrictive in the coming years. Nevertheless, applying 
fairness algorithms often decreases the accuracy of the model and would therefore be a sacrifice for companies in terms of business. One other disadvantage of this work which requires further investigation is to assess the impact of the sensitive bias mitigation on an 
individual level, because it may 
induce negative impacts 
for some individuals. For example, in an extreme case, a person may be refused a position only because of belonging to a privileged group, regardless of their merit within the group. 





\bibliography{mainArXiv}
\bibliographystyle{abbrv}

\newpage
\section{Appendix}
\subsection{Fair adversarial architectures}
\label{sec:Fair_adversarial_architectures}

We illustrate in Figure \ref{fig:Arch_Adv} the various fair adversarial neural network architectures existing in the literature (including this paper). We distinguish two fair adversarial families:
\begin{itemize}
	\item Fair representation: The mitigation is carried on an intermediary latent variable $Z$. The multidimensional latent variable is fed to the adversary and to the predictor.
	\item Prediction retreatment: The mitigation is carried on the prediction itself. The prediction is fed to the adversary.
\end{itemize}
For these two families, we distinguish 3 subfamilies:
\begin{itemize}
	\item Simple Adversarial: The adversary tries to predict the sensitive attribute. The bias is mitigated by fooling this adversary.
	\item Rényi Adversarial: The adversary tries to find adequate non-linear transformations for the estimation of the HGR coefficient. The bias is mitigated via the minimization of this estimation.
	\item MINE Adversarial: The Mutual Information Neural Estimator \cite{belghazi2018mutual}, which relies on the Donsker-Varadhan representation \cite{donsker1983asymptotic} of the Kullback-Leibler divergence, is used as an adversary. The bias is mitigated via the minimization of the mutual information estimation.
\end{itemize}

\begin{figure*}[h!]
  \centerfloat
  \includegraphics[scale=0.65]{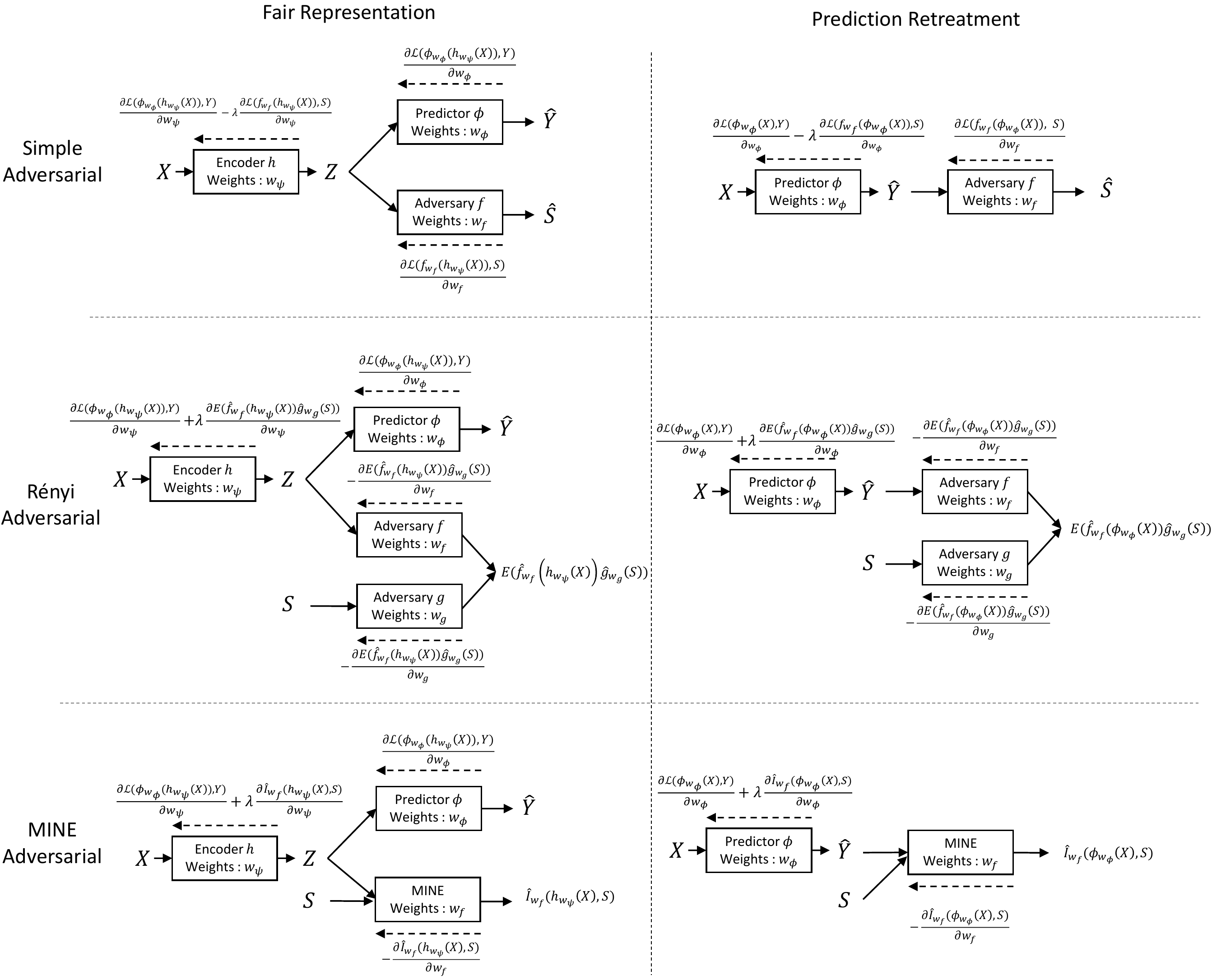}
  \caption{Fair adversarial architectures} 
  \label{fig:Arch_Adv}
\end{figure*}

We represent, in Figure \ref{alg:fair_hgr}, the 6 different fair combinations. Zhang et al. \cite{zhang2018mitigating} feed the prediction output as input to an adversary network that tries to predict the sensitive attribute (upper right). Adel et al. \cite{adel2019one} learn a fair representation by inputting it to the adversary, which is prevented from predicting the sensitive attribute (upper left). Grari et al. \cite{grari2019fairness} minimize the HGR correlation between the prediction output and the sensitive attribute (middle right). Ragonesi et al. \cite{ragonesi2020learning} rely on MINE to minimize the mutual information between a representation and the sensitive attribute (lower left). Another approach, referred to as Fair MINE NN, minimizes the mutual information between the prediction output and the sensitive attribute (lower right). Finally, our algorithm consists in learning a fair representation by minimizing its HGR correlation with the sensitive attribute (middle left). 

\subsection{Proofs}
\label{sec:proofs}
\subsubsection{Consistency of the HGR\_NN}
The domains $\mathcal{U}$ and $\mathcal{V}$ of the random variables $U$ and $V$ are assumed to be compact. \

We define the theoretical HGR as follows:
\begin{eqnarray}
HGR(U, V) = \sup_{\substack{ f:\mathcal{U}\rightarrow \mathbb{R},g:\mathcal{V}\rightarrow \mathbb{R}}}\rho(f(U),g(V))
\end{eqnarray}

where $\rho$ is the Pearson’s correlation coefficient and $f$, $g$ are (measurable) functions with finite and positive variance w.r.t the distributions of $U$ and $V$.

We define the theoretical neural HGR measure associated to a family of neural networks $F_\Theta$:
\begin{equation}
    HGR_{F_\Theta}(U,V) = \sup_{(f_{\theta_f},g_{\theta_g}) \in F_\Theta}  \rho(f_{\theta_{f}}(U),g_{\theta_{g}}(V)) 
\label{opthgr}
\end{equation}
$\Theta$ is a compact domain of $\mathbb{R}^k$ for a given $k$. \\ $F_\Theta \subset \{(f_{\theta_{f}},g_{\theta_{g}}), f_{\theta_{f}}$ and $g_{\theta_{f}}$ neural networks with parameters ($\theta_{f}$,$\theta_{g}$) $\in$ $\Theta$\}. \\ 
We use the abuse of notation $HGR_{\Theta}(U,V)$ to refer to $HGR_{F_\Theta}(U,V)$. $HGR_{\Theta}(U,V)$ is well-defined when $f_{\theta_{f}}(U)$ and  $g_{\theta_{g}}(V)$ are not constant for all ($\theta_{f}$,$\theta_{g}$) $\in$ $\Theta$. \\

We define the empirical HGR neural measure, given $n$ \textit{i.i.d} samples of $(U,V)$ and a family $F_\Theta$, as:
\begin{equation}
    \widehat{HGR(U,V)_n} = \sup_{(\theta_f,\theta_g) \in \Theta}  \rho_n(f_{\theta_{f}}(U),g_{\theta_{g}}(V)) 
\label{emphgr}
\end{equation}

where $\rho_n$ is the sample correlation computed using the samples of $(U,V)$. $\rho_n$ is well-defined \textit{iff} the sample variances are positive.

\begin{lemma}{(approximation)}
\label{approximation}
Let $\eta > 0$. There exists a family of continuous neural networks $F_{\Theta}$ parametrized by a compact domain $\Theta \subset R^{k}$, such that $HGR_{\Theta}(U,V)$ is well-defined and:
\begin{equation}
 \abs{HGR(U,V)-HGR_{\Theta}(U,V)} \leq \eta.
\end{equation}

\end{lemma}
\begin{proof}
Let $\eta>0$ and $\epsilon>0$.
\\
\\
There exist functions $f^{*},g^{*}$ centered and standardized such that: 

\begin{center}
    $HGR(U,V)-\rho(f^{*}(U),g^{*}(V)) < \epsilon$ \\
\end{center}

Let $f$ and $g$ some functions with positive and finite variance and $\Tilde{f}=\frac{(f-\mu_f)}{\sigma_f}$
; $\Tilde{g}=\frac{(g-\mu_g)}{\sigma_g}$, so that $\Tilde{f}(U)$ and $\Tilde{g}(V)$ are centered and standardized.

\begin{subequations}
	\begin{align}
  HGR(U,V)-\rho(f(U),g(V)) &\leq \epsilon + \rho(f^{\ast}(U),g^{\ast}(V)) - \rho(\Tilde{f}(U),\Tilde{g}(V)) \\
 &= \epsilon + E(f^{\ast}(U)g^{\ast}(V)) - E(\Tilde{f}(U)\Tilde{g}(V)) 
	\end{align}
\end{subequations}

Using the Cauchy-Schwarz inequality:
\begin{subequations}
	\begin{align}
  E(f^{\ast}(U)g^{\ast}(V)) - E(\Tilde{f}(U)\Tilde{g}(V)) 
  &= E \left( (f^{\ast}(U)-\Tilde{f}(U))g^{\ast}(V) \right)+E \left((g^{\ast}(V)-\Tilde{g}(V))\Tilde{f}(U)\right) \\
  &\leq \sqrt{E \left((f^{\ast}(U)-\Tilde{f}(U))^2 \right)} + \sqrt{E \Big( (g^{\ast}(V)-\Tilde{g}(V))^2 \Big) }
	\end{align}
\end{subequations}

Let $||h||_{2}=E(h(X)^2)^{1/2}$ with $X \sim U$ or $X \sim V$ depending on the context. The inequality becomes:

\begin{subequations}
	\begin{align}
  HGR(U,V)-\rho(f(U),g(V)) &\leq \epsilon + ||f^{\ast}-\Tilde{f}||_{2} +  ||g^{\ast}-\Tilde{g}||_{2}
	\end{align}
\end{subequations}

Let's find a bound of $||f^{\ast}-\Tilde{f}||_2$ that depends on $||f^{\ast}-f||_2$:
\begin{subequations}
	\begin{align}
  ||f^{\ast}-\Tilde{f}||^2_{2} &= 2-2E\left(f^{\ast}(U)\left(\frac{f(U)-\mu_{f}}{\sigma_f}\right)\right)\\
  &= 2-2\frac{E\left(f^{\ast}(U)f(U)\right)}{\sigma_f}\\
  &= 2+\frac{1}{\sigma_f}E\left((f^{\ast}(U)-f(U))^2-1-\sigma^2_f-\mu^2_f\right)\\
  &\leq 2+\frac{1}{\sigma_f}(||f^{\ast}-f||^2_2-1-\sigma_f^2)\\
  &=\frac{||f^{\ast}-f||^2_2}{\sigma_f}+2-(\frac{1}{\sigma_f}+\sigma_f) \label{bound_eq}
	\end{align}
\end{subequations}

We bound the standard deviation error: \\  
\begin{subequations}
	\begin{align}
    |1-\sigma_f| &\leq \sqrt{|1-E(f(U)^2)|} + |E(f(U))| \\
              &= \sqrt{\Big|E\Big((f^{\ast}(U)-f(U))(f^{\ast}(U)+f(U))\Big)\Big|} + \big|E\big(f(U)-f^{\ast}(U)\big)\big| \\
              &\leq \sqrt{||f^{\ast}-f||_2||f^{\ast}+f||_2} + ||f^{\ast}-f||_2  \\
              &\leq \sqrt{||f^{\ast}-f||_2(||f^{\ast}-f||_2+2||f^{\ast}||_2)} + ||f^{\ast}-f||_2 \\
              &= \sqrt{||f^{\ast}-f||^2_2+2||f^{\ast}-f||_2} + ||f^{\ast}-f||_2 \label{std error}
	\end{align}
\end{subequations}



Using \eqref{std error}, we have:

\begin{subequations}
	\begin{align}
  \frac{||f^{\ast}-f||^2_2}{\sigma_f} \leq 
  \frac{||f^{\ast}-f||^2_2}{1-|1-\sigma_f|} \leq \frac{||f^{\ast}-f||^2_2}{1-(\sqrt{||f^{\ast}-f||_2^2+2||f^{\ast}-f||_2}+||f^{\ast}-f||_2)} 
	\end{align}
\end{subequations}

Combining this with \eqref{bound_eq}:
\begin{subequations}
	\begin{align}
  ||f^{\ast}-\Tilde{f}||^2_2 \leq \frac{||f^{\ast}-f||^2_2}{1-(\sqrt{||f^{\ast}-f||_2^2+2||f^{\ast}-f||_2}+||f^{\ast}-f||_2)}+2 - (\frac{1}{\sigma_f}+\sigma_f)
	\end{align}
\end{subequations}

\begin{subequations}
	\begin{align}
  t:x\rightarrow \frac{x^2}{1-(\sqrt{x^2+2x}+x)}
	\end{align}
\end{subequations}
is continuous at $0$ so there exists $\gamma_1>0$ such that $|x| \leq \gamma_1$ $\Rightarrow$  $t(x)\leq \frac{\eta^2}{8}$

\begin{subequations}
	\begin{align}
  r:x\rightarrow 2-(\frac{1}{x}+x)
	\end{align}
\end{subequations}
is continuous at $1$ so there exists $\gamma_2>0$ such that $|x-1| \leq \gamma_2$ $\Rightarrow$  $r(x)\leq \frac{\eta^2}{8}$

\begin{subequations}
	\begin{align}
  s:x\rightarrow \sqrt{x^2+2x}+x
	\end{align}
\end{subequations}
is continuous at $0$ so there exists $\gamma_3>0$ such that $|x| \leq \gamma_3$ $\Rightarrow$  $|s(x)|\leq \min(\gamma_2,\frac{1}{2})$ \\

By the universal approximation theorem (see corollary 2.2 of \cite{hornik1989multilayer}) and knowing that $U$ is bounded, we may choose a continuous feedforward network function $f_{\theta_{f}}$ such that:
\begin{center}
    $||f^{\ast}-f_{\theta_{f}}||_2 \leq \min(\gamma_1,\gamma_3)$
\end{center}

By construction of $\gamma_1$ and $\gamma_3$, $f_{\theta_{f}}$ has positive variance and: 
$||f^{\ast}-\Tilde{f}_{\theta_{f}}||_2 \leq \sqrt{\frac{\eta^2}{8}+\frac{\eta^2}{8}}=\frac{\eta}{2}$

Similarly, we can choose a continuous feed-forward network function $g_{\theta_{g}}$ such that:
$||g^{\ast}-\Tilde{g}_{\theta_{g}}||_2 \leq \frac{\eta}{2}$ \\

Therefore:
\begin{center}
    $HGR(U,V)-\rho(f_{\theta_f}(U),g_{\theta_g}(V)) \leq \epsilon + \eta$
\end{center}
Taking the limit as $\epsilon$ approaches 0:
\begin{center}
    $HGR(U,V)-\rho(f_{\theta_f}(U),g_{\theta_g}(V)) \leq  \eta$
\end{center}

For $\Theta$ a given subset of $\mathbb{R}^k$ with $k$ the number of coordinates in $(\theta_f,\theta_g)$, we denote as $F_{\Theta}$ the family of neural networks with the same architecture as $(f_{\theta_{f}}, g_{\theta_{g}})$, parametrized by $\Theta$. \\ 

We can find a compact set $\Theta$ containing $(\theta_f,\theta_g)$ such that all the elements of $F_{\Theta}$ have positive and finite variance: while the finitude of the variance is due to the boundedness of $U$, $V$ and the continuity of the neural networks w.r.t the input, the positivity can be obtained by using the argument of the continuity of the variance w.r.t the parameters (due to the boundedness of $U$, $V$ and the continuity of the neural networks w.r.t the parameters).

Choosing such a compact set $\Theta$, we obtain the result:

\begin{equation}
 \abs{HGR(U,V)-HGR_{\Theta}(U,V)} \leq \eta.
\end{equation}

\end{proof}

\begin{lemma}{(estimation)}
\label{estimation}
Let $\eta > 0$, and $F_{\Theta}$  a family of continuous neural networks  parametrized by a compact domain $\Theta \subset R^{k}$. There exists an $N \in \mathbb{N}$ such that:
\begin{equation}
    \forall n \geq N,  
    \abs{\widehat{HGR(U,V)_n}-HGR_\Theta(U,V)} \leq \eta, a.s.
\end{equation}

\end{lemma}

\begin{proof}
To simplify notations, we will note $f$ and $g$ for $f(U)$ and $g(V)$ when there is no ambiguity. \\
Let $\eta > 0$. By triangular inequality:

\begin{equation}
    \abs{\widehat{HGR(U,V)_n}-HGR_\Theta(U,V)} \leq \sup_{(\theta_f,\theta_g) \in \Theta}  \abs[\big]{\rho_n(f_{\theta_{f}},g_{\theta_{g}}) - \rho(f_{\theta_{f}},g_{\theta_{g}})}
    \label{tri_ineq}
\end{equation}

We denote $E_n$ the empirical expectation, so that: 

\begin{equation}
    \rho_n(X,Y) = \frac{E_n(XY) - E_n(X)E_n(Y)}{ \sqrt{E_n(X^2) - E_n(X)^2} \sqrt{E_n(Y^2) - E_n(Y)^2}}
\end{equation}

The function $(\theta_{f},\theta_{g},u,v) \rightarrow (f_{\theta_f}(u),g_{\theta_g}(v))$ is continuous on a compact set, so it is bounded. The neural networks are, therefore, uniformly bounded. The compactness of $\Theta$, along with the uniform boundedness argument and the continuity of the neural networks w.r.t their parameters, allows to use the uniform law of large numbers \cite{geer2000empirical} to obtain the almost sure uniform convergence of all empirical expectations in $\rho_n$, to the corresponding expectations. \\

The almost sure uniform convergence is compatible with addition, subtraction, multiplication and division, so long as some hypotheses are verified. The compatibility with the first three operations can easily be demonstrated. As for division, we rely on the fact that we can find a uniform positive lower bound for $Var(f_{\theta_f})$ and $Var(g_{\theta_g})$. Indeed, these are positive and continuous functions w.r.t $\theta_f$ (\textit{resp.} $\theta_g$) on a compact set. We can note that this uniform positive lower-bound for the variances, combined with the almost sure uniform convergence of the sample variances, allows us to state that, eventually, all sample variances are positive. \\

We deduce, by compatibility of operations with almost sure uniform convergence, the almost sure uniform convergence of $\rho_n(f_{\theta_{f}},g_{\theta_{g}})$ to $\rho(f_{\theta_{f}},g_{\theta_{g}})$. \\

Therefore, by combining the previous result with \eqref{tri_ineq}, we can find $N \in \mathbb{N}$ such that:

\begin{equation}
    \forall n \geq N,  
    \abs{\widehat{HGR(U,V)_n}-HGR_\Theta(U,V)} \leq \eta, a.s.
\end{equation}

\end{proof}

\begin{theorem}
$\widehat{HGR(U,V)_n}$ is strongly consistent.
\end{theorem}

\begin{proof}
This is a direct consequence of Lemma \ref{approximation} combined with Lemma \ref{estimation}.
\end{proof}

\subsubsection{Comparison with simple adversarial algorithms}

\begin{theorem}
If $E(Y|X)$ is constant, then $\sup_{f} \rho(f(X),Y) = 0$. Else,  $f^{*} \in \argmax_{f} \rho(f(X),Y) $ iff there exists $a,b \in \mathbb{R} $, with $a > 0$, such that:
 \begin{align}
     f^{*}(X) = a E(Y|X) + b
 \end{align}
\end{theorem}

\begin{proof}
Let $f$ a function with positive and finite variance w.r.t $X$.
\begin{subequations}
    \begin{align}
        \mathrm{Cov}(f(X),Y) &= E(f(X)Y) - E(f(X))E(Y) \\
                     &= E\Big(E(f(X)Y|X)\Big) - E(f(X))E(E(Y|X)) \\
                     &= E\Big(f(X)E(Y|X)\Big) - E(f(X))E(E(Y|X)) \\
                     &= \mathrm{Cov}(f(X),E(Y|X))
    \end{align}
\end{subequations}
If $E(Y|X)$ is constant, $\mathrm{Cov}(f(X),Y) = 0$ and therefore $\rho(f(X),Y) = 0$, so that $\sup_{f} \rho(f(X),Y) = 0$. Else, by the Cauchy-Schwarz inequality: 

\begin{subequations}
    \begin{align}
        \rho(f(X),Y) &= \frac{\mathrm{Cov}(f(X),Y)}{\sigma_{f(X)}\sigma_Y} \\
                     &= \frac{\mathrm{Cov}(f(X),E(Y|X))}{\sigma_{f(X)}\sigma_Y} \\
                     &\leq \frac{\sigma_{E(Y|X)}}{\sigma_Y} \label{c-s_inequality} \\
                     &= \rho(E(Y|X),Y)
    \end{align}
\end{subequations}

The inequality above shows that any linear transformation of $E(Y|X)$ with positive slope maximizes $\rho(f(X),Y)$. Conversely, for $f^{*} \in \argmax_{f} \rho(f(X),Y)$, \ref{c-s_inequality} is an equality, which gives $\rho(f^{*}(X),E(Y|X)) = 1$. This implies that there exists $a,b \in \mathbb{R} $, with $a > 0$, such that $f^{*}(X) = a E(Y|X) + b$.
\end{proof}

Note that a one-dimensional linear regression with $f^{*}(X)$ as input and $Y$ as output allows to find $E(Y|X)$.
\begin{proposition}{}
\label{conditional_expectation}
Given $Y \sim \mathcal{N}(\mu, \sigma^{2})$, $X = \arctan(Y^{2}) + U\pi $, where  $U \perp Y$  and $U$ follows a Bernoulli distribution with  $p = \frac{1}{2}$, we have: 
\begin{center}
     $E(Y|X) = \tanh \Big(\frac{\mu}{\sigma^2}\sqrt{\tan(X)}\Big)\sqrt{\tan(X)}$
\end{center}
\end{proposition}

\begin{proof}
We have $Y^2 = \tan(X)$, so that:
\begin{align}
     Y = (2 \mathds{1}_{\{Y>0\}} - 1)\sqrt{\tan(X)} 
     \label{equality_indic}
\end{align}
so it is sufficient to compute $E(\mathds{1}_{\{Y>0\}}|X)$:
\begin{subequations}
\begin{align}
    E(\mathds{1}_{\{Y>0\}}|X) &= E\Big(E(\mathds{1}_{\{Y>0\}}|X,U)\Big|X\Big) \\
                              &= E\Big(E(\mathds{1}_{\{Y>0\}}|\tan(X),U)\Big|X\Big) \\
                              &= E\Big(E(\mathds{1}_{\{Y>0\}}|\tan(X))\Big|X\Big) \\
                              &= E(\mathds{1}_{\{Y>0\}}|\tan(X)) \\
                              &= E(\mathds{1}_{\{Y>0\}}|Y^2)
\end{align} 
\end{subequations}

Let $y > 0$ and $0 < \epsilon < y$:
\begin{subequations}
\begin{align}
    E(\mathds{1}_{\{Y>0\}}||Y^2-y|< \epsilon) 
    &= \frac{\mathbb{P}(Y >0, \sqrt{y-\epsilon} < Y < \sqrt{y+\epsilon})}{\mathbb{P}( \sqrt{y-\epsilon} < Y < \sqrt{y+\epsilon}) + \mathbb{P}( -\sqrt{y+\epsilon} < Y < -\sqrt{y-\epsilon})} \\
    &= \frac{\displaystyle \int_{\sqrt{y-\epsilon}}^{\sqrt{y+\epsilon}} P_Y(u) du}{\displaystyle \int_{\sqrt{y-\epsilon}}^{\sqrt{y+\epsilon}} P_Y(u) du + \int_{-\sqrt{y+\epsilon}}^{-\sqrt{y-\epsilon}} P_Y(u) du} \\
    &= \frac{\displaystyle \int_{\sqrt{y-\epsilon}}^{\sqrt{y+\epsilon}} P_Y(u) du}{\displaystyle \int_{\sqrt{y-\epsilon}}^{\sqrt{y+\epsilon}} (P_Y(u) + P_Y(-u)) du} \\
    &\xrightarrow[\epsilon \to 0]{} \frac{P_Y(\sqrt{y})}{P_Y(\sqrt{y}) + P_Y(-\sqrt{y})}
\end{align} 
\end{subequations}

Therefore, knowing that $P_Y(y) = \frac{e^{-\frac{1}{2}(\frac{y-\mu}{\sigma})^2}}{\sqrt{2\pi}}$ we have:
\begin{subequations}
\begin{align}
    2E(\mathds{1}_{\{Y>0\}}|X) - 1 &= \frac{P_Y(|Y|)-P_Y(-|Y|)}{P_Y(|Y|)+P_Y(-|Y|)} \\
                                   &= \tanh \Big(\frac{\mu}{\sigma^2}|Y|\Big) \\
                                   &= \tanh \Big(\frac{\mu}{\sigma^2}\sqrt{\tan(X)}\Big)
    \label{conditional_exp_indic}
\end{align}
\end{subequations}

Taking the conditional expectation in \ref{equality_indic} and plugging in \ref{conditional_exp_indic}, we obtain:
\begin{center}
     $E(Y|X) = \tanh \Big(\frac{\mu}{\sigma^2}\sqrt{\tan(X)}\Big)\sqrt{\tan(X)}$
\end{center}
\end{proof}
\begin{proposition}
\label{bounds_correlation}
With the same hypotheses as in proposition \ref{conditional_expectation}, and denoting $\alpha = \frac{\mu}{\sigma}$, we have:
\begin{center}
    $\sqrt{1-e^{-\frac{\alpha^2}{2}}}\leq \rho(E(Y|X),Y) \leq \sqrt{1-e^{-\frac{\alpha^2}{2}}(1+\alpha^2)^{-\frac{3}{2}}}$
\end{center}
\end{proposition}
\begin{proof}
We first note that, knowing that $|Y| = \sqrt{\tan(X)}$ and with a parity argument:
\begin{align}
    E(Y|X) = \tanh \Big(\frac{\mu}{\sigma^2}Y \Big)Y
\end{align}
We have:
\begin{subequations}
\begin{align}
    \rho(E(Y|X),Y)^2 &= \frac{\mathrm{Cov}(Y,E(Y|X))}{\sigma^2} \\
                     &= \frac{\mathrm{Cov}(Y,Y)-\mathrm{Cov}(Y,Y - E(Y|X))}{\sigma^2} \\
                     &= 1 - E\left(\left(\frac{Y}{\sigma}\right)^2\Big(1-\tanh\Big(\frac{\mu}{\sigma^2}Y\Big)\Big)\right)
                     \label{rho_expression}
\end{align}
\end{subequations}
With a variable change ($y = \frac{y'}{\sigma}$), we obtain: 
\begin{subequations}
\begin{align}
    E\left(\left(\frac{Y}{\sigma}\right)^2\Big(1-\tanh\Big(\frac{\mu}{\sigma^2}Y\Big)\Big)\right)
    &=  \frac{1}{\sqrt{2\pi}} \int_{\mathbb{R}} \frac{2y^2e^{-\alpha y}}{e^{\alpha y}+ e^{-\alpha y}} e^{-\frac{1}{2}(y-\alpha)^2}dy \\
    &= e^{-\frac{\alpha^2}{2}} \times \frac{1}{\sqrt{2\pi}} \int_{\mathbb{R}} \frac{y^2}{\cosh(\alpha y)} e^{-\frac{1}{2}y^2} dy
    \label{expectation_integral}
\end{align}
\end{subequations}
We have, for all $y \in \mathbb{R}$, $1 \leq \cosh(\alpha y) \leq e^{\frac{\alpha^2 y^2}{2}}$. This gives:
\begin{align}
    \frac{1}{\sqrt{2\pi}} \int_{\mathbb{R}} y^2 e^{-\frac{1}{2}y^2} dy 
    \leq \frac{1}{\sqrt{2\pi}} \int_{\mathbb{R}} \frac{y^2}{\cosh(\alpha y)} e^{-\frac{1}{2}y^2} dy 
    \leq \frac{1}{\sqrt{2\pi}} \int_{\mathbb{R}} y^2 e^{-\frac{1}{2}(1+\alpha^2)y^2} dy
\end{align}
i.e 
\begin{align}
    1 \leq \frac{1}{\sqrt{2\pi}} \int_{\mathbb{R}} \frac{y^2}{\cosh(\alpha y)} e^{-\frac{1}{2}y^2} dy \leq (1+\alpha^2)^{-\frac{3}{2}}
    \label{integral_ineq}
\end{align}

We combine \ref{rho_expression}, \ref{expectation_integral} and \ref{integral_ineq} to obtain the result:
\begin{center}
    $\sqrt{1-e^{-\frac{\alpha^2}{2}}}\leq \rho(E(Y|X),Y) \leq \sqrt{1-e^{-\frac{\alpha^2}{2}}(1+\alpha^2)^{-\frac{3}{2}}}$
\end{center}
\end{proof}
\vspace{-0.5cm}
\begin{figure*}[h!]
  \centerfloat
  \includegraphics[scale=0.6]{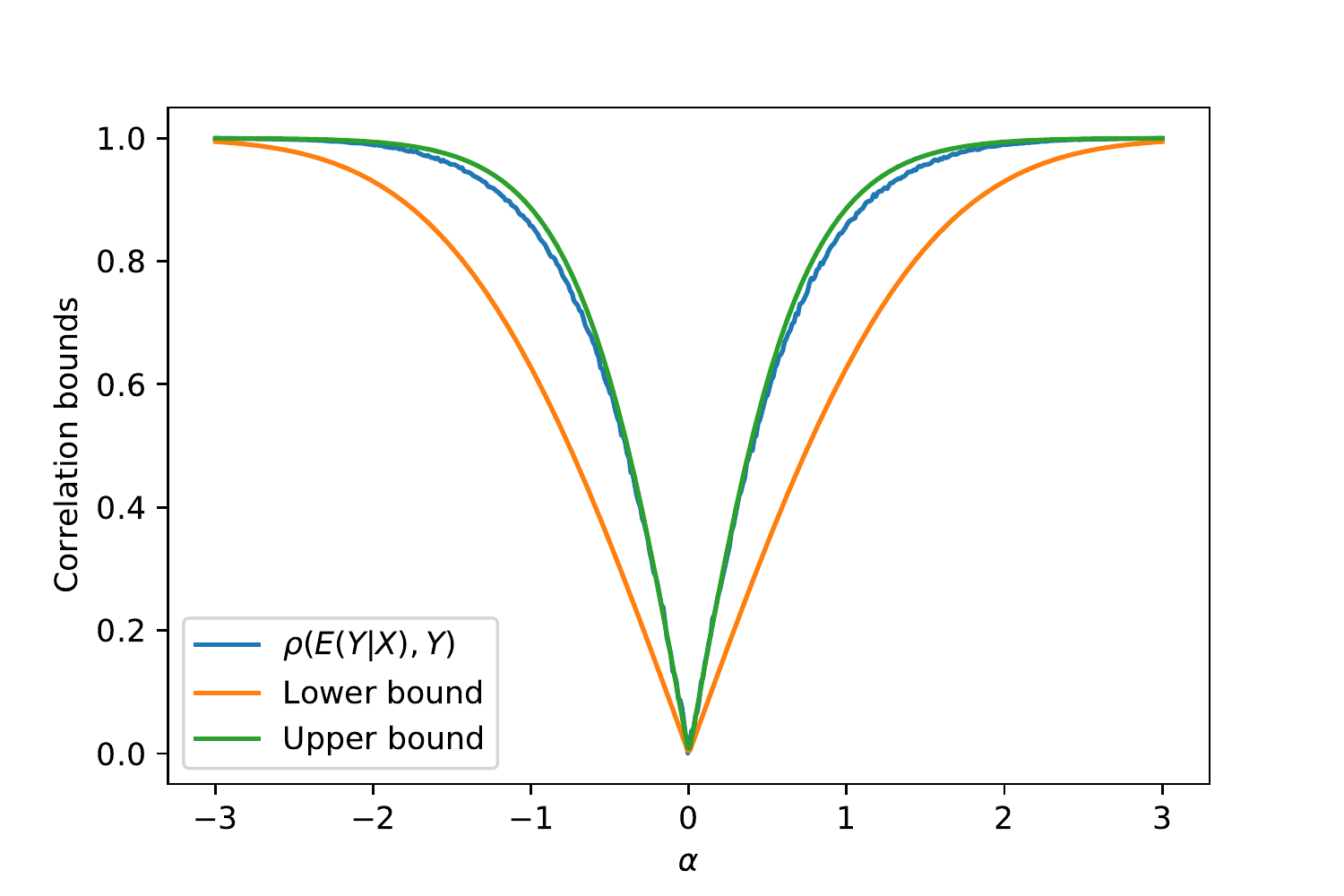}
  \caption{Simplified HGR correlation w.r.t $\alpha$} 
  \label{fig:correlation_bounds}
\end{figure*}

In Figure \ref{fig:correlation_bounds}, we illustrate the bounds found in proposition \ref{bounds_correlation}, $\rho(E(Y|X),Y)$ being estimated by Monte-Carlo. First, we note that the upper bound is close to $\rho(E(Y|X),Y)$, whereas the lower bound $\sqrt{1-e^{-\frac{\alpha^2}{2}}}$ is not as precise. For non-zero values of $\alpha$, $\rho(E(Y|X),Y)$ is positive, so that a predictive neural network can capture some non-linear dependencies between $Y$ and $X$. This is due to the fact that, for $\alpha \neq 0$, the square function is bijective when restricted to some open interval containing the mean of $Y$, whereas when $\alpha = 0$, such an interval cannot be found. When this interval is large and the standard deviation of $Y$ is not too large (which corresponds to high values of $|\alpha|$), $\rho(E(Y|X),Y)$ approaches 1 and the $Y$ prediction error approaches 0. In the opposite case, $\rho(E(Y|X),Y)$ is close to 0 and a predictive neural network cannot capture dependencies. 
\\
\begin{proposition}
We consider the global fairness objective of the prediction retreatment simple adversarial algorithm, with $X$ the input data, $Y$ the output data and $S$ the sensitive attribute (with $\widehat{Y} = f(X)$):
\begin{align}
    \max_{f} \min_{g} E\Big(\big(S-g(f(X))\big)^2\Big)
\end{align}
whose 
optimum is achieved when $E(S \vert \widehat{Y})=E(S)$, different from the  demographic parity fairness objective $P(S \vert \widehat{Y})=P(S)$ for continuous features.
\end{proposition}

\begin{proof}
We have:
\begin{center}
    $\max_{f} \min_{g} E\Big(\big(S-g(f(X))\big)^2\Big) = \max_{f} E\Big(\big(S-E(S|f(X))\big)^2\Big)$
\end{center}

Some algebraic manipulations with expectations give:
\begin{subequations}
\begin{align}
   E\Big(\big(S-E(S|\widehat{Y})\big)^2\Big) &= E(S^2) - 2E\Big(SE(S|\widehat{Y})\Big) + E(E(S|\widehat{Y})^2) \\
                      &= E(S^2) - 2E\Big(E\big(SE(S|\widehat{Y})\big|\widehat{Y}\big)\Big) + E(E(S|\widehat{Y})^2) \\
                      &= E(S^2) - E(E(S|\widehat{Y})^2) \\
                      &= (E(S^2) - E(S)^2) - \Big(E(E(S|\widehat{Y})^2) - E(E(S|\widehat{Y}))^2\Big) \\
                      &= \sigma_{S}^2 - \sigma_{E(S|\widehat{Y})}^2
\end{align}

Therefore, the global fairness objective is equivalent to 
\begin{center}
    $\min_{f} \sigma_{E(S|f(X))}^2$
\end{center}
In the optimal case, we have $\sigma_{E(S|\widehat{Y})} = 0$, which corresponds to the case when $E(S|\widehat{Y})$ is constant equal to its expectation i.e:
\begin{center}
    $E(S \vert \widehat{Y})=E(S)$
\end{center}
\end{subequations}

\end{proof}



\subsection{Algorithm}
\label{sec:algorithm}

\begin{algorithm*}[h]
\caption{Fair Representation via HGR NN}
\label{alg:fair_hgr}
\begin{algorithmic}

\STATE \textbf{Input:}
Training set ${\cal T}$, 
Loss function $\mathcal{L}$, Batchsize $b$, 
\\
{\color{white} \textbf{Input:   }} Neural Networks $h_{w_{\psi}}$,$\phi_{w_{\phi}}$, $f_{w_f}$ and $g_{w_g}$,  $\qquad$ $\qquad$ $\quad$ 
\\
{\color{white} \textbf{Input:   }}  Learning rates $\alpha_f$, $\alpha_g$, $\alpha_\phi$ and $\alpha_\psi$. Fairness control $\lambda$ 
\STATE \textbf{Repeat}
\STATE Draw $b$ samples $(x_{1}, s_1, y_{1}), . . . ,(x_{b}, s_b, y_{b})$ from ${\cal T}$
\STATE Compute the predictor objective:
\STATE $L_Y(w_{\phi},w_{\psi}) = \frac{1}{b}\sum_{i=1}^{b} \mathcal{L}(\phi_{w_{\phi}}(h_{w_{\psi}}(x_i)),y_i)$
\STATE Update the predictor model $\phi_{w_{\phi}}$ by gradient descent: \\
$w_{\phi} \leftarrow w_{\phi} - \alpha_\phi
(\frac{\partial L_Y}{\partial {w_\phi}})$
\STATE Calculate the mean and variance of the transformations:
\STATE $m_{f} \leftarrow \frac{1}{b}\sum_{i=1}^{b}f_{w_f}(h_{w_{\psi}}(x_i))$ ; $m_{g} \leftarrow \frac{1}{b}\sum_{i=1}^{b}g_{w_g}(s_{i})$  
\STATE $\sigma_{f}^{2} \leftarrow \frac{1}{b}\sum_{i=1}^{b}(f_{w_f}(h_{w_{\psi}}(x_i))-m_{f})^2$
\STATE $\sigma_{g}^{2} \leftarrow \frac{1}{b}\sum_{i=1}^{b}(g_{w_g}(s_{i})-m_{g})^2$ \\
\STATE Standardize the transformations:
\STATE $\forall i: \hat{f}_{w_f}(h_{w_{\psi}}(x_i)) \leftarrow  \frac{f_{w_f}(h_{w_{\psi}}(x_i))-m_{f}}{\sqrt{\sigma_{f}^{2}+\epsilon}} $
\STATE $\forall i: \hat{g}_{w_g}(s_i) \leftarrow  \frac{g_{w_g}(s_i)-m_{g}}{\sqrt{\sigma_{g}^{2}+\epsilon}}$
\STATE Compute the objectives:
\STATE $J(w_{f},w_{g}, w_{\psi})=\frac{1}{b}\sum_{i=1}^{b}\hat{f}_{w_f}(h_{w_{\psi}}(x_i))*\hat{g}_{w_g}({s}_{i})$
\STATE $L_E(w_{\phi},w_{\psi},w_{f},w_{g}) = \frac{1}{b}\sum_{i=1}^{b} \mathcal{L}(\phi_{w_{\phi}}(h_{w_{\psi}}(x_i)),y_i) + \lambda J(w_{f},w_{g}, w_{\psi})$
\STATE Update the adversary by gradient ascent: 
\STATE $w_{f} \leftarrow w_{f} + \alpha_{f}\frac{\partial J}{\partial w_{f}}$; \ $w_{g} \leftarrow w_{g} +\alpha_{g}\frac{\partial J}{\partial w_{g}}$ 
\STATE Update the encoder model $h_{w_{\psi}}$ by gradient descent:
\STATE $w_\psi \leftarrow w_\psi -\alpha_{\psi} (\frac{\partial L_E}{\partial {w_\psi}})$
\end{algorithmic}
\end{algorithm*}
Algorithm \ref{alg:fair_hgr} depicts our Fair HGR NN algorithm for the Demographic Parity task. 
The algorithm takes as input a training set composed of triplets $(x_i,s_i,y_i)$. At each iteration, it samples batches of size $b$ from the training data and updates the predictor parameters $w_\phi$ by one step of gradient descent with the learning rate $\alpha_\phi$. Second, it standardizes the outputs of networks $f_{w_f}$ and $g_{w_g}$ to ensure 0 mean and a variance of 1 on the batch. Then, it computes the HGR\_NN objective function, which corresponds to the empirical correlation, to estimate the HGR coefficient and the global objective. Finally, at the end of each iteration, the algorithm updates the parameters of the adversary $w_f$ and $w_g$ by one step of gradient ascent and the encoding parameters $w_\psi$ by one step of gradient descent. Back-propagation is performed on the full architecture, including mean and variance calculations, to avoid oscillations.

\subsection{Experiments}
\label{sec:experiments}

\subsubsection{Data sets}

Our experiments on real-world data are performed on five data sets. First, we experiment with three data sets where the sensitive and the outcome true value are both continuous:

\begin{itemize}
\item The US Census demographic data set \cite{USCensus} is an extraction of the 2015 American Community Survey, 
with  37 features about 74,000 census tracts. The target is the 
percentage of children below the poverty line,  
the sensitive attribute is the percentage of women in the census tract.
\item The Motor Insurance data set \cite{pricinggame15} originates from a pricing game organized by The French Institute of Actuaries in 2015,  
with 15 attributes for 36,311 observations. The target is the average claim cost 
per policy, 
the sensitive attribute is  the driver's age. 
\item The Crime data set is obtained from the UCI Machine Learning Repository  
\cite{Dua:2019}, 
with 128 attributes for 1,994 instances. 
The target is the number of violent crimes per population, 
the sensitive attribute 
is the ratio of an ethnic group per population.
\end{itemize}

We experiment with two data sets with a binary classification task where the sensitive features are continuous:
\begin{itemize}
\item Compas: The COMPAS data set \cite{angwin2016machine} contains 13 attributes of about 7,000 convicted criminals with class labels that
state whether or not the individual reoffended within 2 years of their most recent crime. Here, we use age as sensitive attribute.

\item Default: The Default data set~\cite{Yeh:2009:CDM:1464526.1465163} contains 23 features about 30,000 Taiwanese credit card users with class labels which state whether an individual will default on payments. As sensitive attribute, we use age.

\end{itemize}


\subsubsection{Experimental parameters} 

For the reproducibility of the experimental results, we reported the deep learning architecture and the different hyperparameters chosen. For all data sets, we repeat five experiments by randomly sampling two subsets, 80\% for the training set and 20\% for the test set.


Since the different data sets are not large, we train the different algorithms on a  NVIDIA Titan Xp (12 Gb) GPU and we report the average runtime of each scenario (Runtime (s)). Note that we use an Adam optimization for each scenario.

\begin{table}[h!]
\resizebox{\textwidth}{!}{
\begin{tabular}{rccccccc}
\hline
Scenario &  $\lambda$  &  Nb Epochs  & Batch Size & Architecture $h_{w_{\psi}}$ &  Architecture $\phi_{w_{\phi}}$  & Architecture $f_{w_{f}}$ \& $g_{w_{g}}$ & Runtime (s) 
\\ \hline
Biased Model & 0  & 200  & 2048  & FC:16 R, FC:8 R, FC:2  & FC:16 R, FC:8 R, FC:4 R, FC:1 Sig &  FC:64 R, FC:64 R, FC:1 &  303 \\
Biased Model & 13  & 200  & 2048  & FC:16 R, FC:8 R, FC:2  & FC:16 R, FC:8 R, FC:4 R, FC:1 Sig &  FC:64 R, FC:64 R, FC:1 &  287 \\

\hline
\centering
\label{reproducibility}
\end{tabular}
}
\caption{Synthetic Scenario. FC stands for fully connected, R for the ReLU activation function and Sig for the Sigmoid activation function.}
\end{table}

\begin{table}[h!]
\resizebox{\textwidth}{!}{
\begin{tabular}{rccccccc}
\hline
Scenario &  $\lambda$  &  Nb Epochs  & Batch Size & Architecture $h_{w_{\psi}}$ &  Architecture $\phi_{w_{\phi}}$  & Architecture $f_{w_{f}}$ \& $g_{w_{g}}$ & Runtime (s)
\\ \hline
$\sigma \leq$ 0.03 & 0.250  & 10  & 512  & see Table \ref{encoderMNIST} &  FC:10 SM &  FC:64 R, FC:64 R, FC:1 & 326 \\
$\sigma >$ 0.04     & 0.100  & 10  & 512 &  see Table \ref{encoderMNIST} &  FC:10 SM  & FC:64 R, FC:64 R, FC:1 &  371  \\

\hline
\centering
\label{reproducibility}
\end{tabular}
}
\caption{MNIST with Continuous Color Intensity. FC stands for fully connected, R for ReLU, SM for the Softmax activation function.}
\end{table}
\begin{table}[h!]
\resizebox{\textwidth}{!}{
\begin{tabular}{rccccccc}
\hline
Scenario &  $\lambda$  &  Nb Epochs  & Batch Size & Architecture $h_{w_{\psi}}$ &  Architecture $\phi_{w_{\phi}}$  & Architecture $f_{w_{f}}$ \& $g_{w_{g}}$ & Runtime (s) 
\\ \hline
US Census & 20 & 150 & 2048  & FC:128 R, FC:64 R,FC:64  & FC:128 R, FC:64 R,FC:16 R,FC:1  & FC:64 R, FC:64 R, FC:1  &  1873 \\
Motor     & 1.5  & 1000  & 2048 & FC:128 R, FC:64 R,FC:64  & FC:128 R, FC:64 T,FC:16 R,FC:1  & FC:64 R, FC:64 T, FC:1  &  235  \\
Crime     & 3  &  3000  & 512 & FC:128 R, FC:64 R,FC:64  & FC:128 R, FC:64 T,FC:16 R,FC:1  & FC:64 R, FC:64 T, FC:1  & 1584  \\
COMPAS    & 200  & 850  & 2048& FC:128 R, FC:64 R,FC:64  & FC:128 R, FC:64 R,FC:16 R,FC:1 Sig & FC:64 R, FC:64 R, FC:1  & 1721 \\
Default   & 100  & 400  & 2048  & FC:128 R, FC:64 R,FC:64  & FC:128 R, FC:64 R,FC:16 R,FC:1 Sig & FC:64 R, FC:64 R, FC:1  & 3378 \\
\hline
\centering
\label{reproducibility}
\end{tabular}
}
\caption{Real-world Experiments. FC stands for fully connected, T for Tanh, R for the ReLU activation function and Sig for the Sigmoid activation function.}
\end{table}

\begin{table}[h!]
\resizebox{\textwidth}{!}{
\begin{tabular}{rccccccc}
\toprule
\multicolumn{1}{c}{} & \multicolumn{5}{c}{Encoder MNIST $h_{w_{\psi}}$}      \\
\hline
Layer &  Number of outputs  &  Kernel size & Stride & Activation function   
\\ \hline
Input $x$ & $3*28*28$  &   &   &   &    \\
Convolution  & $64*26*26$ & $5*5$  & 1  &  ReLU   \\
MaxPooling  & $64*13*13$ & - & 2  &  -   \\
Convolution  & $64*11*11$  & $5*5$  & 1  &  ReLU   \\
MaxPooling  & $64*5*5$ & - & 2  &  -   \\
Flatten &- & -& -& - \\
Fully-connected & 512  & -  & -  &  ReLU   \\
Fully-connected & 64  & -  & -  &  None   \\

\hline
\centering
\label{encoderMNIST}
\end{tabular}
}
\caption{Encoder $h_{w_{\psi}}$ used for the MNIST Scenario with Continuous Color Intensity}
\end{table}

\end{document}